\documentclass{article}
\usepackage[a4paper, total={6in, 8in}]{geometry}
\usepackage[utf8]{inputenc}
\usepackage{graphicx}
\usepackage{caption}
\usepackage{subcaption}
\usepackage{indentfirst}
\usepackage{amsmath,amssymb,amsthm}
\usepackage{listings}
\usepackage{graphics}

\usepackage{amsmath,amssymb,amsthm}
\usepackage[ruled,vlined]{algorithm2e}
\usepackage{caption}
\usepackage{subcaption}

\usepackage{xurl}
\usepackage[colorlinks,allcolors=blue]{hyperref}

\theoremstyle{definition}
\newtheorem{proposition}{Proposition}
\newtheorem{property}{Property}

\theoremstyle{definition}
\newtheorem{deff}{Definition}

\theoremstyle{definition}

\usepackage{color}
\usepackage[dvipsnames]{xcolor}

\title{Vertex-based reachability analysis for verifying ReLU deep neural networks}
\author{João Zago\footnote{joao.zago@posgrad.ufsc.br}, Eduardo Camponogara\footnote{eduardo.camponogara@ufsc.br}, Eric Antonelo\footnote{eric.antonelo@ufsc.br }}
\date{January, 2023}

\begin{document}

\maketitle

\begin{abstract}
Neural networks achieved high performance over different tasks, i.e. image identification, voice recognition and other applications. Despite their success, these models are still vulnerable regarding small perturbations, which can be used to craft the so-called adversarial examples.
Different approaches have been proposed to circumvent their vulnerability, including formal verification systems, which employ a variety of techniques, including reachability, optimization and search procedures, to verify that the model satisfies some property.
In this paper we propose three novel reachability algorithms for verifying deep neural networks with ReLU activations.
The first and third algorithms compute an over-approximation for the reachable set, whereas the second one computes the exact reachable set. 
Differently from previously proposed approaches, our algorithms take as input a V-polytope.
Our experiments on the ACAS Xu problem show that the Exact Polytope Network Mapping (EPNM) reachability algorithm  proposed in this work surpass the state-of-the-art results from the literature, specially in relation to other reachability methods.
\end{abstract}

\section{Introduction}
\label{sec:introduction}

Regardless of the success of deep neural networks in computer vision and natural language processing, these models are susceptible to small perturbations applied to their inputs, i.e. it is possible to misguide the model output by applying a designed perturbation to a given input. For instance, the ACAS Xu model \cite{julian2019deep} (explained later in detail), that responded differently from expected while under specific circumstances. The inputs purposefully designed to force a misbehavior of the neural network are denoted as adversarial examples \cite{szegedy2013intriguing}.

To overcome such vulnerability many different approaches have been previously employed. One proposed the application of algorithms that were able to generate adversarial examples, the so called adversarial attacks, and subsequently applied these inputs in the training process of the network \cite{szegedy2013intriguing, goodfellow2014explaining, madry2017towards, tramer2017ensemble}. There were also those approaches that aim to identify the adversarial examples before feeding them as input to the neural network \cite{song2017pixeldefend, grosse2017statistical, metzen2017detecting, lu2017safetynet}.
Even though these procedures helped to reduce the vulnerability of the neural networks, these models remained vulnerable to adversarial attacks.

Formal methods were also applied to certify or guarantee that the model behaves as expected under some circumstances or within a specified domain region, nevertheless \cite{katz2017reluplex} showed that the verification problem is NP-hard, leaving the process of certifying large models still an open problem.

The existing formal procedures can be classified into three different categories: 1) reachability methods; 2) optimization methods; and 3) search methods. The first one relies on calculating the output mapping of an input set \cite{xiang2017reachable, xiang2018output, gehr2018ai2}, the second one comprises the application of mathematical optimization (Mixed Integer-Linear Programming or Convex Optimization) to identify counter-examples \cite{lomuscio2017approach, tjeng2017evaluating, dvijotham2018dual, wong2018provable}, and the third makes use of both reachability and optimization approaches in conjunction with search methods for identifying counter-examples \cite{katz2017reluplex, huang2017safety}.

In this paper, we propose three novel reachability algorithms: APNM and PAPNM algorithms that compute an over-approximation for the output, while EPNM which computes the exact mapping. 
We present demonstrations on the behavior and correctness of these algorithms and case studies of their applications for comparison with existing algorithms from the literature. We also show that the algorithms proposed in this work are highly parallelizable.

The rest of this paper is organized into five sections: Section \ref{sec:realated_works} gives an overview of the existing algorithms and related works;
Section \ref{sec:algorithms} describes the proposed algorithms;
Section \ref{sec:demonstrations} addresses the demonstrations regarding the completeness and soundness of the proposed algorithms;
Section \ref{sec:applications} presents a study case for application and comparison;
and finally Section \ref{sec:conclusion} discusses the outcome of the procedures presented in this work.

\section{Related Works}
\label{sec:realated_works}

Formal verification of neural networks is receiving a huge amount of attention mainly because of its importance in security sensitive tasks such as the application of neural networks in autonomous vehicles, systems controllers, aeronautics, and several other applications that can possibly involve financial, human or environmental injury \cite{huang2017safety, julian2019deep}.

As previously presented, there are three main research areas developed to provide guarantees for neural networks: 1) reachability methods; 2) optimization methods; and 3) search methods. In this work, we propose three novel approaches regarding reachability analysis, which consist of two major steps: computing the output set by mapping a subset of the neural network domain and comparing the output set with the specification.

One of the algorithms developed in the literature to verify neural networks by means of reachability analysis is denoted as MaxSens \cite{xiang2018output}, which computes an over-approximation to the output reachable set and compares it to the desired specification. As this algorithm computes an over-approximation it does not satisfy the completeness property of a verification algorithm, although it is sound (if the property is verified then the algorithm guarantees that it is satisfied). 
To compute such an over-approximation, MaxSens divides the input set into several smaller hyperrectangles and computes the maximum sensitivity for each of them layer-by-layer. The output of this algorithm consists of the union of several hyperrectangles, which approximates the exact reachable set. Note that their approach can be applied to neural networks with any activation function.

Another approach from the literature, denoted as ExactReach \cite{xiang2017reachable}, computes the exact reachable set given an input set. This procedure takes as inputs a convex H-polytope (a polytope represented by its inequalities) and computes the exact mapping layer-by-layer. The authors proposed this approach specifically for the ReLU activation function, which is reasonable as this function has achieved promising results for convolutional and fully connected neural networks, which are widely applied in different applications.
Due to the nature of the ReLU activation function, the authors separated the non-linear mapping process in three cases: 1) all the elements of the input are positive; 2) all the elements of the input are negative; and 3) the input has positive, negative or null elements. These cases cover all the mapping possibilities regarding the ReLU activation.

Similarly to the aforementioned approaches, we propose in this paper three novel algorithms for reachability analysis. However, instead of using the H-polytope representation, each of our proposed procedures take as input set a V-polytope (a polytope represented by its vertices). 
We demonstrate the correctness of both algorithms and compare them with the literature (not only with those that make use of reachability analysis) by using a case study (ACAS Xu \cite{julian2019deep}).

\section{Algorithms}
\label{sec:algorithms}

In this section we will state the verification problem and the three algorithms proposed in this work. The demonstrations regarding the correctness of each procedure will be presented in the following section.

\subsection{Problem statement}

Let $F: \mathbb{R}^n \mapsto \mathbb{R}^m$ represent a mapping given by a neural network composed of $L$ layers.
Then, for a given $\mathbf{x} \in \mathbb{R}^n$, $F(\mathbf{x}) = (F_L \circ F_{L-1} \circ \cdots \circ F_2 \circ F_1)(\mathbf{x})$, 
where $F_l$ is the mapping given by the $l$-th layer.
Further, $F_l$ consists of the composition of an affine mapping and a non-linear mapping (in this case a ReLU function): $F_l(\mathbf{x}_l) = ReLU(\mathbf{W}_l \mathbf{x}_l + \boldsymbol{\theta}_l)$, where $\mathbf{W}_l$ and $\boldsymbol{\theta}_l$ denote the weight matrix and bias vector of the $l$-th layer, respectively, and $\mathbf{x}_l$ is the input to the same layer.

Suppose that $\mathcal{X} \subseteq \mathbb{R}^n$ is the input set that we want to verify and $\mathcal{R} = \{F(\mathbf{x}) \mid \mathbf{x} \in \mathcal{X}\}$ is the exact output set associated with $\mathcal{X}$, which resulted from the application of the neural network $F$ to each input in $\mathcal{X}$. 
Moreover, the set $\mathcal{Y}$ comprises the expected output set for the inputs from $\mathcal{X}$. Then, the verification problem consists of assuring that:
\begin{equation}
    \mathcal{R} \cap \lnot \mathcal{Y} = \emptyset.
    \label{eq:verification_property}
\end{equation}
In other words, we want to guarantee that there will not exist an input of $F$ in $\mathcal{X}$ that will cause the network to generate an undesired output (an output that is not in $\mathcal{Y}$).

To perform such a verification, one needs to start by calculating the output set associated with $\mathcal{X}$. As we presented in previous sections, there are approaches that compute the exact reachable set and other approaches that over-approximate it. For those that compute an approximation for the output set, denoted by $\widehat{\mathcal{R}}$, we will have that $\mathcal{R} \subseteq \widehat{\mathcal{R}}$. Then, we can see that if $\widehat{\mathcal{R}} \cap \lnot \mathcal{Y} = \emptyset$, then the property given by the Equation \ref{eq:verification_property} will still be assured. In the following sections we will present the proposed algorithms.

\subsection{Approximate Polytope Network Mapping (APNM)}

The first algorithm proposed in this work is called Approximate Polytope Network Mapping (APNM)
and consists of a procedure to compute an over-approximation for the reachable set. 
Let $\mathcal{P}$ be a convex closed polytope, defined as a convex combination of its vertices, namely 
\begin{equation} 
\mathcal{P} = \left\{\sum_{i=1}^o \lambda_i \mathbf{v}_i \mathrel{\bigg|} \sum_{i=1}^o \lambda_i = 1, \lambda_i \geq 0 \text{ and } \mathbf{v}_i \in \mathcal{V}, \forall i \in \{1, \ldots, o\} \right \},
\end{equation} where $\mathcal{V}$ is the set of vertices of $\mathcal{P}$ ($\mathcal{V} = \{\mathbf{v}_1, \ldots, \mathbf{v}_o\}$).
To achieve its goal, the algorithm has five parts:

\begin{enumerate}
    \item Affine map of the vertices with weights and biases;
    \item Adjacent vertex identification;
    \item Polytope intersection with an orthant's hyperplanes;
    \item ReLU mapping; and
    \item Removing non-vertices.
\end{enumerate}

In what follows, the way the algorithm works will be explained with a simple $2$-dimensional problem for visualization purposes. The input set $\mathcal{P}$ is presented in Figure \ref{fig:input_set}: the set of vertices is $\mathcal{V} = \{\mathbf{v}_1, \mathbf{v}_2, \mathbf{v}_3, \mathbf{v}_4\}$, where $\mathbf{v}_1 = (1.0,1.0)$, $\mathbf{v}_2 = (-1.0,1.0)$, $\mathbf{v}_3 = (-1.0,-1.0)$ and $\mathbf{v}_4 = (1.0,-1.0)$; and the set of edges is $\mathcal{E}=\{(\mathbf{v}_1,\mathbf{v}_2), (\mathbf{v}_1,\mathbf{v}_4), (\mathbf{v}_2,\mathbf{v}_3), (\mathbf{v}_3,\mathbf{v}_4)\}$.

\begin{figure}[!ht]
    \centering
    \includegraphics[scale=0.3]{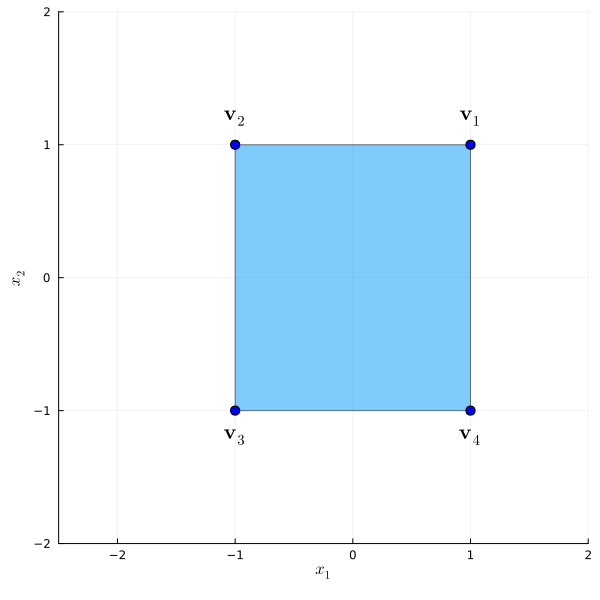}
    \caption{The input set for visualizing each step of the algorithm.}
    \label{fig:input_set}
\end{figure}

\subsubsection*{Affine map of the vertices with weights and biases}

The affine map comprises the product of each vertex of $\mathcal{P}$ by the weights associated with layer $l$ plus  the biases of the same layer. Algorithm \ref{alg:affine_map} contains the steps to compute the affine map of all vertices from $\mathcal{V}$. 


\begin{algorithm}[h!]
\SetAlgoLined
 \textbf{Input:} $\mathbf{V} \in \mathbb{R}^{o \times n}$, $\mathbf{V} = \left[\mathbf{v}[1], \ldots, \mathbf{v}[o]\right]$, $\mathbf{v}[i] \in \mathbb{R}^n \, , \, \forall i \in \{1, \ldots, o\}$, $\mathbf{W} \in \mathbb{R}^{m \times n}$ and $\boldsymbol{\theta} \in \mathbb{R}^{m}$

 \SetKwFunction{FMain}{$\mathbf{AM}$}
 \SetKwProg{Fn}{Function}{:}{}
 \Fn{\FMain{$\mathbf{V}, \mathbf{W}, \boldsymbol{\theta}$  }}
  {
 %
 
 let $\mathbf{Z}[1..o]$ be a new array, where each $\mathbf{Z}[i] \text{ is an empty array}, \, \forall i \in \{1, \ldots, o\}$
 
 \For{$i \in \{1, \ldots, o\}$}{
    $\mathbf{Z}[i] = \mathbf{W} \mathbf{v}[i] + \boldsymbol{\theta}$ \;
 }
 
 \textbf{return} $\mathbf{Z}$ 
 }
 
 \caption{Affine Map (AM)}
 \label{alg:affine_map}
\end{algorithm}

By applying the affine map to the input set, presented in Figure \ref{fig:input_set}, we have the output of this step as shown in Figure \ref{alg:affine_map}. The weight matrix $\mathbf{W}$ and biases $\boldsymbol{\theta}$ employed in this toy example are as follows:
\begin{equation}
\mathbf{W} = 
\begin{bmatrix}
    0.492693 & -1.29232 \\
    0.925861 & 0.675146
\end{bmatrix},
\quad
\boldsymbol{\theta} = 
\begin{bmatrix}
    -0.18857972 \\
    -0.14839205
\end{bmatrix}
\end{equation}
Note that, as expected, the output of the current operation is a simple affine transformation of the vertices of $\mathcal{P}$.

\begin{figure}[!ht]
    \centering
    \includegraphics[scale=0.3]{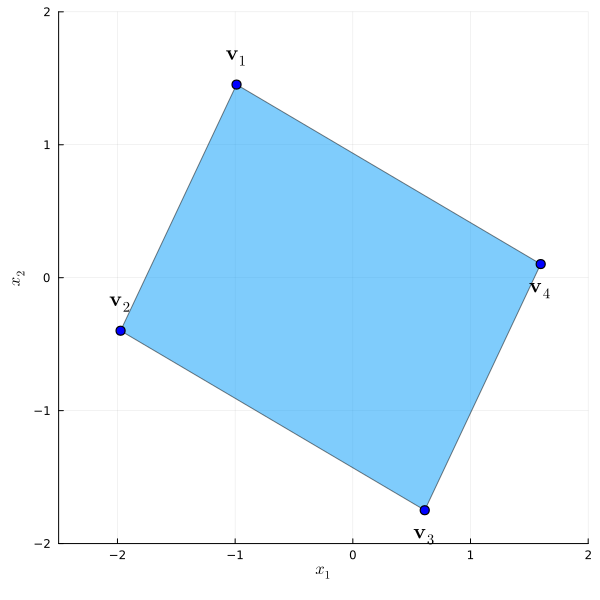}
    \caption{Representation of the output of the affine map operation.}
    \label{fig:affine_map}
\end{figure}

\subsubsection*{Adjacent vertex identification}

Following the affine map, the edge identification process takes place. 
Let $\mathbf{v}_i,\mathbf{v}_j\in\mathcal{V}$ be two distinct vertices, such that $i \neq j$.
If there is at least one combination of the vertices of $\mathcal{V}$, except from $\mathbf{v}_i$ or $\mathbf{v}_j$, that allow us to compute the middle point of $\mathbf{v}_i$ and $\mathbf{v}_j$, given by $(\mathbf{v}_i + \mathbf{v}_j)/2$, then $\mathbf{v}_i$ and $\mathbf{v}_j$ are not adjacent.
Equation \ref{eq:adjascent_vertices} presents this feasibility problem as a MILP (Mixed Integer-Linear Programming), where $\boldsymbol{\lambda} \in \mathbb{R}^o$ is the vector used to represent the convex combination of the vertices of $\mathcal{P}$ and $\eta \in \{0,1\}$ is the binary variable that enables evaluating $\mathbf{v}_i$ or $\mathbf{v}_j$ separately.

\begin{subequations}
\begin{eqnarray}
\max && 0 \, \text{,} \\
\text{s.t.} && \frac{(\mathbf{v}_i + \mathbf{v}_j)}{2} \, = \, \sum_{k=1}^o \mathbf{v}_k \lambda_k \, \text{,} \\
&& \sum_{k=1}^o \lambda_k \, = \, 1 \, \text{,} \\ 
&& \lambda_k \, \geq \, 0 \, \text{,} \forall k \in \{1, \ldots, o\} \, \text{,} \\ 
&& \lambda_i \, \leq \, \eta \, \text{,} \\ 
&& \lambda_j \, \leq \, 1 - \eta \, \text{,} \\ 
&& \eta \, \in \, \{0,1\} \\ \nonumber
\end{eqnarray}
\label{eq:adjascent_vertices}
\end{subequations}

Algorithm \ref{alg:edge_identification} presents the identification 
of adjacent vertices. 
It creates an undirected graph using an adjacency list as the data structure to represent the $1$-skeleton of the polytope, which is the edge structure of some polytope \cite{mcmullen2002abstract, emiris2016efficient}. We denote the adjacency list of the undirected graph by $\mathbf{E}$, were each element is a set containing the 
indices
of the adjacent vertices.

\begin{algorithm}
\SetAlgoLined
 \textbf{Input:} $\mathbf{V} \in \mathbb{R}^{o \times n}$, $\mathbf{V} = [\mathbf{v}[1], \ldots, \mathbf{v}[o]]$, $\mathbf{v}[i] \in \mathbb{R}^n \, , \, \forall i \in \{1, \ldots, o\}$
 

 \SetKwFunction{FMain}{$\mathbf{EI}$}
 \SetKwProg{Fn}{Function}{:}{}
 \Fn{\FMain{$\mathbf{V}$  }}
 {
 let $\mathbf{E}[1 \ldots o]$ be a new array, where each $\mathbf{E}[i] \, , \, \forall i \in \{1, \ldots, o\}$, corresponds to an empty set
 
 \For{$i \in \{1, \ldots, o\}$}{
 
    \For{$j \in \{i+1, \ldots, o\}$}{
    
        \uIf{$\lnot \exists \boldsymbol{\lambda} \in \mathbb{R}^o : \frac{1}{2}(\mathbf{v}[i] + \mathbf{v}[j]) = \sum_{k=1}^o \mathbf{v}[k] \lambda[k] \land \sum_{k=1}^o \lambda[k] = 1 \land \lambda[k] \geq 0 \, , \, \forall k \in \{1, \ldots, o\} \land (\lambda[i] = 0 \lor \lambda[j] = 0)$}{
        
            $\mathbf{E}[i] \gets \mathbf{E}[i] \cup \{j\}$
        }
    
    }
 
 }

 \textbf{return} $\mathbf{E}$ 

 }
 
 \caption{Edge-skeleton Identification (EI)}
 \label{alg:edge_identification}
\end{algorithm}

Following the toy problem, where the affine map is shown in Figure \ref{fig:affine_map}, the undirected graph that represents the edge structure of the affine map of $\mathcal{P}$ is presented in Figure \ref{fig:graph_viz}. This graph $\mathcal{G}=(\mathcal{V},\mathcal{E})$ consists of the vertex set $\mathcal{V} = \{\mathbf{v}_1, \mathbf{v}_2, \mathbf{v}_3, \mathbf{v}_4\}$ and edge set  $\mathcal{E}=\{(\mathbf{v}_1,\mathbf{v}_2), (\mathbf{v}_1,\mathbf{v}_4), (\mathbf{v}_2,\mathbf{v}_3), (\mathbf{v}_3,\mathbf{v}_4)\}$.

Notice that each edge in the undirected graph indicates that the associated vertices are connected by an edge of the corresponding polytope.

\begin{figure}[htbp]
    \centering
    \includegraphics[scale=0.07]{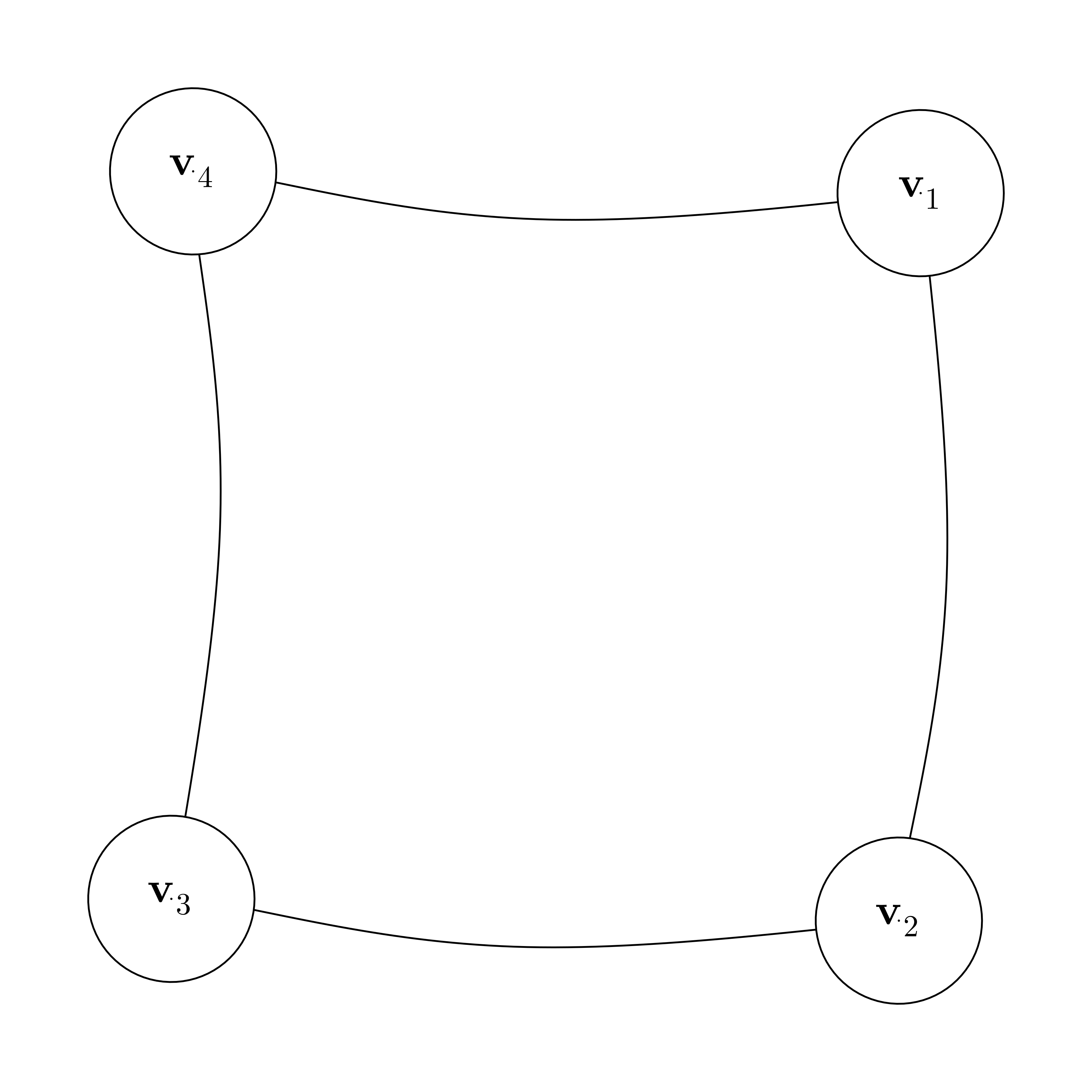}
    \caption{Undirected graph that represents the edge structure of the output of the affine map of $\mathcal{P}$.}
    \label{fig:graph_viz}
\end{figure}

\subsubsection*{Polytope intersection with orthant's hyperplanes}

The next step for the algorithm concerns the identification of those points at the intersection of the hyperplanes that define each orthant with the edges of $\mathcal{P}$.
As the edges of the polytope were previously computed, only the edges whose extreme points belong to different orthants need to be identified, i.e., extreme points that have at least one component with a different sign. So, we compute the difference of the sign for each element of both extreme points for a given edge, given by Equation \ref{eq:diff_vertex_signal}:
\begin{equation}
    \mathbf{a} = {\tt sign}(\mathbf{v}_{i}) - {\tt sign}(\mathbf{v}_{j})
    \label{eq:diff_vertex_signal}
\end{equation}
where $i$ denotes the vertex under consideration, $j \in \mathbf{E}_i$ represents the index of the vertices that are adjacent to $\mathbf{v}_i$, and ${\tt sign}:\mathbb{R}^n \rightarrow \mathbb{R}^n$ is the signal mapping that associates $1$ for the positive elements, $-1$ for the negative elements, and $0$ for the null ones. We also define a function $\sigma:\mathbb{R} \rightarrow \mathbb{R}$, given by Equation \ref{eq:sigma}. For the cases where a component of $\mathbf{a}$ is greater than or equal to $2$, we have that the sign indeed changed. For those cases where the difference is equal to $1$, there was a null component in one of the vertices, which does not indicate that they are in different orthants.
\begin{equation}
    \sigma(\alpha) = 
    \begin{cases}
        1,& \text{if } |\alpha| \geq 2\\
        0,& \text{otherwise}
    \end{cases}
    \label{eq:sigma}
\end{equation}

Then, for those cases where there is a sign change, or, in other words, for those cases that satisfy the condition:
\begin{equation}
    \sum_{k=1}^n  \sigma(a_{k}) \neq 0
\end{equation}
we compute those points at which the intersection between the edge and the hyperplane occurred. Each sign change generates an intersection point.

Given two vertices $\mathbf{v}_i$ and $\mathbf{v}_j$, suppose that there is a $k \in \{1, \ldots, n\}$ for which $\sigma(a_{k}) \neq 0$, indicating that $\mathbf{v}_i$ and $\mathbf{v}_j$ belong to different orthants. So, there exists a value of $\lambda$, such that $0 \leq \lambda \leq 1$, defining a convex combination of the vertices that lies on the orthant hyperplane, namely a $\lambda$ that satisfies:
\begin{equation}
  \lambda v_{j,k} + (1-\lambda)v_{i,k} = 0 \iff  \lambda = \frac{-v_{i,k}}{v_{j,k} - v_{i,k}}
    \label{eq:lambda}
\end{equation}
Figure \ref{fig:lambda_view} contains a visual representation of this process for the case of vertices $\mathbf{v}_1=(-0.988207,1.45261)$ and $\mathbf{v}_4=(1.59643,0.102323)$. Notice that $\mathbf{a} = {\tt sign}(\mathbf{v}_1) - {\tt sign}(\mathbf{v}_4) = (-1,1) - (1,1) = (-2,0)$, thus $\boldsymbol{\sigma}(\mathbf{a}) = (1,0)$ implying that $\mathbf{v}_1$ and $\mathbf{v}_4$ lie in different orthants, and the edge connecting them intercepts the hyperplane defining the respective orthant at a point $\mathbf{p}$. By using the above equation, we compute the convex combination parameter $\lambda$ to be $0.382338$. 

\begin{figure}[htbp]
    \centering
    \includegraphics[scale=0.3]{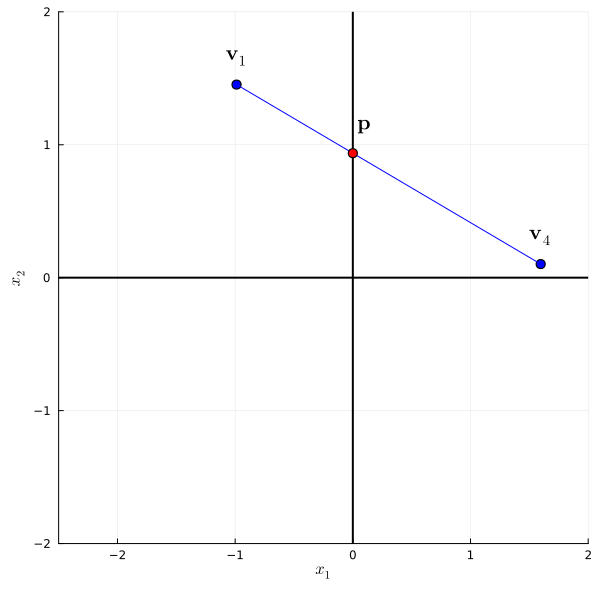}
    \caption{Representation of the intersection identification process. The blue dots represent two adjacent vertices and the line segment is the edge connecting them. The red dot is the intersection point of the edge with respect to the orthant hyperplane supporting plane.}
    \label{fig:lambda_view}
\end{figure}

Then, we compute each component $s \in \{1, \ldots, n\}$ of the point $\mathbf{p} \in \mathbb{R}^n$, as presented by Equation \ref{eq:componentes_calculation},
\begin{equation}
    p_s = (v_{j,s} - v_{i,s}) \lambda + v_{i,s},
    \label{eq:componentes_calculation}
\end{equation}
which represents the intersection between an edge of $\mathcal{P}$ and one of the hyperplanes of an orthant from $\mathbb{R}^n$.
For the example in Figure \ref{fig:lambda_view}, the intercept of the edge $(\mathbf{v}_1,\mathbf{v}_4)$ with the supporting hyperplane $\{\mathbf{x}=(x_1,x_2):x_1= 0\}$ of the
orthant $\{\mathbf{x}:\mathbf{x}\geq 0\}$ is $\mathbf{p}=(0,0.93635)$.

The procedure, formalized by Algorithm \ref{alg:orthant_intersection}, computes the intersection points between each edge of $\mathcal{P}$ and each orthant's supporting hyperplane that are intercepted by the edge.

\begin{algorithm}[htb!]
\SetAlgoLined
 \textbf{Input:} $\mathbf{V} \in \mathbb{R}^{o \times n}$ where $\mathbf{V} = [\mathbf{v}[1], \ldots, \mathbf{v}[o]]$, $\mathbf{v}[i] \in \mathbb{R}^n \, , \, \forall i 
 \in \{1, \ldots, o\}$, \\
 \qquad\quad $\mathbf{E} = [\mathbf{E}[1], \ldots, \mathbf{E}[o]]$
 

 \SetKwFunction{FMain}{$\mathbf{II}$}
 \SetKwProg{Fn}{Function}{:}{}
 \Fn{\FMain{$\mathbf{V},\mathbf{E}$}}
 {
 
 \For{$i \in \{1, \ldots, o\}$}{

 \tcc{For each edge of vertice i} 
    \For{$j \in \mathbf{E}[i]$  
    }{   
 
        $\mathbf{a} \gets {\tt sign}(\mathbf{v}[i]) - {\tt sign}(\mathbf{v}[j])$
     
            
            \For{$k \in \{1,\ldots,n\}$}{
            
                \uIf{$\sigma(a[k]) \neq 0$}{
                    
                    $\lambda \gets \frac{-v[i,k]}{v[j,k] - v[i,k]}$
                    
                    let $\mathbf{p}$ be a point in $\mathbb{R}^n$
                    
                    \For{$s \in \{1, \ldots, n\}$}{
                    
                        $p[s] \gets \frac{v[j,s] - v[i,s]}{v[i,s]} \lambda$
                    
                    }
                    
                    $\mathbf{V} \gets [\mathbf{V}\, |\, \mathbf{p}]$ 
                    \tcc{Append $\mathbf{p}$ to matrix $\mathbf{V}$ }
                    
                
                }
            }    
    }
    
 }
 
 \textbf{return} $\mathbf{V}$ 

 }
 \caption{Orthant intersection identification (II)}
 \label{alg:orthant_intersection}
\end{algorithm}

The result of the intersection identification for the output of the affine map of $\mathcal{P}$ is presented in Figure \ref{fig:intersection_identification}. There we present all 
the vertices of the polytope in addition to the computed intersection points.

\begin{figure}[!ht]
    \centering
    \includegraphics[scale=0.3]{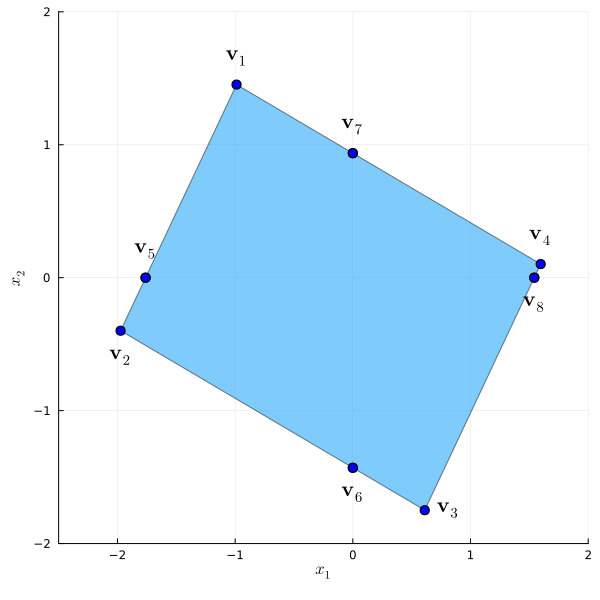}
    \caption{Intersection identification results representation.}
    \label{fig:intersection_identification}
\end{figure}

\subsubsection*{ReLU Mapping}

After calculating  the intersection points between edges and orthant's hyperplanes, the ReLU mapping of the resulting points is computed. We apply the mapping $ReLU:\mathbb{R}^n \rightarrow \mathbb{R}^n$, defined by Equation \ref{eq:relu_map}, to all the points given by the previous procedure. Algorithm \ref{alg:relu_map} summarizes the process where $\mathbf{V}$ is the output vertex set from Algorithm \ref{alg:orthant_intersection}.
\begin{equation}
    ReLU(\mathbf{x}) = \max(0,\mathbf{x})
    \label{eq:relu_map}
\end{equation}

\begin{algorithm}[htb!]
\SetAlgoLined
 \textbf{Input:} $\mathbf{V} \in \mathbb{R}^{o \times n}$, $\mathbf{V} = [\mathbf{v}[1], \ldots, \mathbf{v}[o]]$, $\mathbf{v}[i] \in \mathbb{R}^n \, , \, \forall i \in \{1, \ldots, o\}$
 
\SetKwFunction{FMain}{$\mathbf{ReLU}$}
 \SetKwProg{Fn}{Function}{:}{}
 \Fn{\FMain{$\mathbf{V}$  }}
 {
 
 \For{$i \in \{1, \ldots, o\}$}{
 
    \For{$j \in \{1, \ldots, n\}$}{
    
        $v[i,j] \gets \max(0, v[i,j])$
        
    }
 
 }
 
 \textbf{return} $\mathbf{V}$ 

 }
 \caption{ReLU map}
 \label{alg:relu_map}
\end{algorithm}

We can see the result of the application of the $ReLU$ mapping as presented in Figure \ref{fig:relu_map}. As expected, all the vertices were projected to the positive orthant, whereby vertices $\mathbf{v}_2$,  $\mathbf{v}_5$, and $\mathbf{v}_6$ are projected onto the origin. Notice that the output set is not convex.

\begin{figure}[!ht]
    \centering
    \includegraphics[scale=0.3]{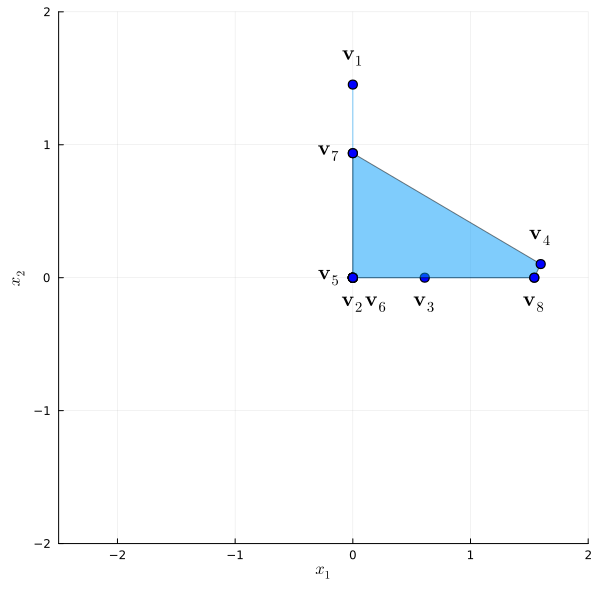}
    \caption{Representation of the output of the $ReLU$ mapping of the vertices and intersection points.}
    \label{fig:relu_map}
\end{figure}

\subsubsection*{Removing Non-vertices}

To simplify the output of a single layer, algorithm APNM computes the convex hull on the output of the ReLU mapping. To compute the convex hull, the algorithm removes the points that are not vertices. To create a generalized process for eliminating non-vertices, and due to the high dimensionality of the internal layers of the neural network, we propose the application of a feasibility analysis to identify the points that can be expressed as a convex combination from the remaining points.
   
The feasibility problem is formalized in Equation \ref{eq:convex_hull}, as follows:
\begin{subequations}
\begin{eqnarray}
\min && 0 \, \text{,} \\
\text{s.t.} && \sum_{i=1}^o \lambda_i \mathbf{v}_i \, = \, \mathbf{v}_k \, \text{,} \\
&& \sum_{i=1}^o \lambda_i \, = \, 1 \, , \\
&& \lambda_k \, = \, 0 \, , \\
&& \lambda_i \, \geq \, 0 \, , \, \forall i \in \{1, \ldots, o\},  \\ \nonumber
\end{eqnarray}
\label{eq:convex_hull}
\end{subequations}
where, for each vertex candidate, denoted by $\mathbf{v}_k$, we search for a $\boldsymbol{\lambda}$ vector such that $\mathbf{v}_k$ can be described as a convex combination of the remaining vertices. If that is the case, then $\mathbf{v}_k$ is not a vertex and it can be removed. Otherwise, the point is a vertex and must remain in the set of vertices.
The non-vertex elimination is formalized in Algorithm \ref{alg:removing_internal_points}.

\begin{algorithm}[htb!]
\SetAlgoLined
 \textbf{Input:} $\mathbf{V} \in \mathbb{R}^{o \times n}$, $\mathbf{V} = [\mathbf{v}[1], \ldots, \mathbf{v}[o]]$, $\mathbf{v}[i] \in \mathbb{R}^n \, , \, \forall i \in \{1, \ldots, o\}$
 
 \SetKwFunction{FMain}{$\mathbf{RP}$}
 \SetKwProg{Fn}{Function}{:}{}
 \Fn{\FMain{$\mathbf{V}$  }}
 {
 
 \For{$k \in \{1, \ldots, o\}$}{
    
    \uIf{$\lnot \exists \boldsymbol{\lambda} \in \mathbb{R}^o : \sum_{i=1}^o \lambda_i \mathbf{v}[i] = \mathbf{v}[k] \land \sum_{i=1}^o \lambda_i = 1 \land \lambda_k = 0 \land \lambda_i \geq 0 \, , \, \forall i \in \{1, \ldots, o\}$}{
        $\mathbf{V} \gets [\mathbf{v}[1], \ldots, \mathbf{v}[k-1], \mathbf{v}[k+1], \ldots, \mathbf{v}[o]]$
    }
    
 
 }
 
 \textbf{return} $\mathbf{V}$
 }
 \caption{Removing Internal Points (RP)}
 \label{alg:removing_internal_points}
\end{algorithm}

The output of the process of removing non-vertices  can be viewed in Figure \ref{fig:identifying_non_vertices}. As we can see, instead of returning the exact non-convex set, this algorithm computes an over-approximation for the output set which is simpler than the exact output as it is a unique set.

\begin{figure}[!ht]
    \centering
    \includegraphics[scale=0.3]{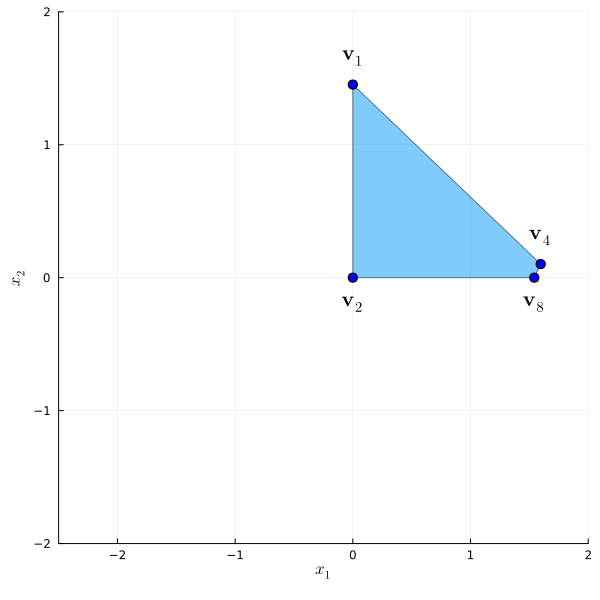}
    \caption{Representation of the convex hull of the output of the $ReLU$ mapping.}
    \label{fig:identifying_non_vertices}
\end{figure}


\subsubsection*{Approximate Polytope Network Mapping}

Finally, we show the complete process, presented by the Algorithm \ref{alg:layer_map}, to compute an approximation for the output reachable set for a given input polyhedron $\mathcal{P}$.

\begin{algorithm}[htb!]
\SetAlgoLined
    \textbf{Input:} $\mathbf{V} \in \mathbb{R}^{o \times n}$, $\mathbf{W}_l$ and $\boldsymbol{\theta}_l$ for $l \in \{1, \ldots, L\}$
 
     \SetKwFunction{FMain}{$\mathbf{SAPNM}$}
     \SetKwProg{Fn}{Function}{:}{}
     \Fn{\FMain{$\mathbf{V}, \mathbf{W}, \boldsymbol{\theta}$  }}
     {
    
    $\mathbf{Z}[0] \gets \mathbf{V}$
 
    \For{$l \in \{1, \ldots, L\}$}{
    
        $\hat{\mathbf{Z}}[l] \gets \textbf{AM}(\mathbf{Z}[l-1], \mathbf{W}_l, \boldsymbol{\theta}_l)$ \tcc{Compute affine mapping for layer $l$}
        
        \uIf{$l < L$}{
        
            $\mathbf{E}[l] \gets \textbf{EI}(\hat{\mathbf{Z}}[l])$ \tcc{Compute edge-skeleton of polyhedron given by vertices $\hat{\mathbf{Z}}[l]$}
            
            $\hat{\mathbf{Z}}[l] \gets \textbf{II}(\hat{\mathbf{Z}}[l], \mathbf{E}[l])$ \tcc{Obtain intercept of edges with horthant hyperplanes}
            
            $\hat{\mathbf{Z}}[l] \gets \textbf{ReLU}(\hat{\mathbf{Z}}[l])$ \tcc{Apply ReLU mapping to the vertex set $\hat{\mathbf{Z}}[l]$}
            
            $\mathbf{Z}[l] \gets \textbf{RP}(\hat{\mathbf{Z}}[l])$  \tcc{Remove non-vertex points from vertex set $\mathbf{Z}[l]$}
            
        }
    }
 
 \textbf{return} $\mathbf{Z}[L]$

  }
 \caption{Simple Approximated Politope Network Mapping}
 \label{alg:layer_map}
\end{algorithm}

\subsection{Exact polytope network mapping (EPNM)}

We present in this section the second proposed approach, which similarly to the first one computes the reachable set by using the vertices of a given input polytope, but computes the exact output instead. We consider the same input closed convex polytope $\mathcal{P}$ and the same set of vertices $\mathcal{V}$ as presented in the previous section. This approach comprises six parts:
\begin{enumerate}
    \item Affine map of the vertices with weights and biases;
    \item Adjacent vertex identification;
    \item Origin verification;
    \item Polytope intersection with orthant's hyperplanes;
    \item Separate points according to the orthant to which they belong;
    \item ReLU mapping; and
    \item Removing non-vertices;
\end{enumerate}
In this section we only present parts 3 and 5, which were not previously introduced.

\subsubsection*{Origin verification}

After we compute the edges of $\mathcal{P}$, we need to verify if the origin is inside it. 
In case it is true, then we need to insert it into the set of vertices that will be partitioned with respect to the orthant that they belong in the next part of the algorithm. 
This verification is compulsory due to the separation of vertices according to the orthant to which they belong (for instance, if the origin is in $\mathcal{P}$ and is not in the set of the vertices of the partitions of the polytope, then the union of the partitions will not be equal $\mathcal{P}$), as presented in Figure \ref{fig:zeros_verification}. 
Notice that, each partition represent the portion of $\mathcal{P}$ that is inside an orthant, i.e., the portion of $\mathcal{P}$ that is in the positive orthant, in Figure \ref{fig:zeros_verification}, is the union of the red area and the shaded area.
We can see that if the origin is not part of the set of vertices of each partition of $\mathcal{P}$, after the orthant separation process, the shaded area will be missing for the portion of $\mathcal{P}$ inside the positive orthant. 
Thus, we need to include the origin point in the respective set so that the portion of $\mathcal{P}$ in the positive orthant also includes the shaded area.
\begin{figure}[htbp]
    \centering
    \includegraphics[scale=0.65]{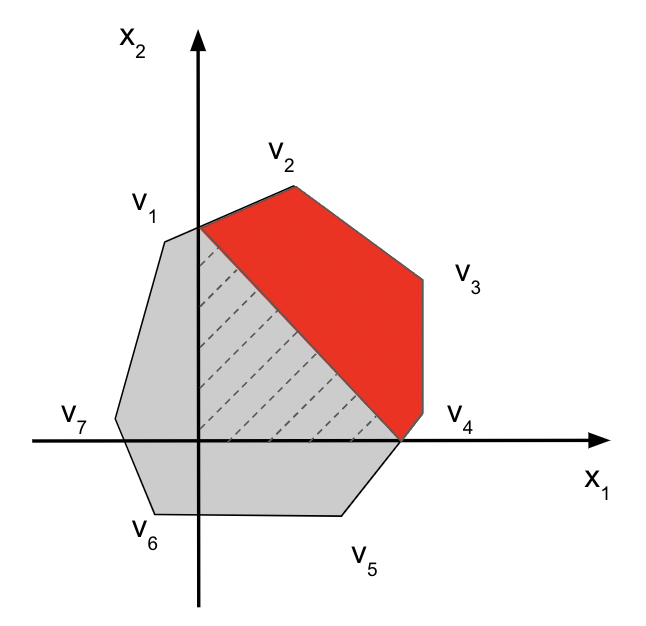}
    \caption{Representation of the issue that will occur if the origin is not inserted to the vertex set for the orthant separation process.}
    \label{fig:zeros_verification}
\end{figure}
To perform such a verification, we state the search problem as a feasibility analysis, in which we try to verify if there exists a vector $\boldsymbol{\lambda} \in \mathbb{R}^o$ such that the origin is a convex combination of the vertices in $\mathcal{V}$. This problem is formalized by Equation \eqref{eq:origin_search}.

\begin{subequations} \label{eq:origin_search}
\begin{eqnarray}
\min && 0 \, \text{,} \\ 
\text{s.t.} && \sum_{i=1}^o \lambda_i \mathbf{v}_i \, = \, 0 \, \text{,} \\
&& \sum_{i=1}^o \lambda_i \, = \, 1 \, , \\
&& \lambda_i \, \geq \, 0 \, , \, \forall i \in \{1, \ldots, o\} \text{.} \\ \nonumber
\end{eqnarray}
\end{subequations}

Algorithm \ref{alg:origin_search} formalizes the search process associated with the verification of the origin.

\begin{algorithm}[htb!]
\SetAlgoLined
    \textbf{Input:} $\mathbf{V}\in \mathbb{R}^{o \times n}$, $\mathbf{V} = [\mathbf{v}[1], \ldots, \mathbf{v}[o]]$, $\mathbf{v}[i] \in \mathbb{R}^n \, , \, \forall i \in \{1, \ldots, o\}$
 
    \SetKwFunction{FMain}{$\mathbf{OS}$}
    \SetKwProg{Fn}{Function}{:}{}
    \Fn{\FMain{$\mathbf{V}$  }}
    {
    
    \uIf{$\exists \boldsymbol{\lambda} \in \mathbb{R}^o : \sum_{i=1}^o \lambda_i \mathbf{v}_i = 0 \land \sum_{i=1}^o \lambda_i = 1 \land \lambda_i \geq 0 \, , \, \forall i \in \{1, \ldots, o\}$}{
        $\mathbf{V} \gets [\mathbf{v}[1], \ldots, \mathbf{v}[o] \, | \, \mathbf{0}]$ 
        \tcc{Adding the origin to the vertex set}
    }
 
    \textbf{return} $V$

    }
    \caption{Origin Search (OS)}
    \label{alg:origin_search}
\end{algorithm}

\subsubsection*{Separating points to their respective orthants}

We present in this section the process of splitting the vertex set $\mathcal{V}$ into several sets, each associated with a different orthant.
In other words, we split $\mathcal{V}$ in such a way that each resulting set contains points located in a single orthant (observe that the convex hull of each set of vertices represents the partition of $\mathcal{P}$ that is inside an orthant).
We begin by computing:
\begin{equation}
    \mathbf{b} = \frac{1}{2}({\tt sign}(\mathbf{v}_i) + \mathbf{1}_{|\mathbf{v}_i|})
    \label{eq:sign_separation}
\end{equation}
where, $\mathbf{v}_i$ is the vertex under analysis, ${\tt sign}: \mathbb{R}^n \rightarrow \mathbb{R}^n$ is the sign mapping previously presented and $\mathbf{1}_{|\mathbf{v}_i|}$ is a vector with all the components equal to one and has size $|\mathbf{v}_i|$.

The elements of $\mathbf{b}$ with value equal to $1$ are associated with a positive element of $\mathbf{v}_i$, the elements equal to $0$ are associated with negative components of the vertex, and those components of $\mathbf{b}$ equal to $1/2$ are associated with the null components of the vertex. Notice that a null component of a vertex means that this point belongs to at least two different orthants (the origin being a special case inside all the orthants).



For those cases where the component of $\mathbf{b}$ is equal to $1/2$, we must guarantee that the associated vertex $\mathbf{v}_i$ is properly inserted into each of those sets that are associated with the orthants it belongs to (i.e., if $\mathbf{v}_i$ has one null component, it must be inserted in two sets).  
Then, the algorithm starts a verification process to identify those component of $\mathbf{b}$ that are equal $1/2$.
In case it is equal true, two copies of $\mathbf{v}_i$ must be created.
This verification process is repeated until all components of $\mathbf{b}$ have been checked. This process is detailed in Algorithm \ref{alg:zeros_verification}.

\begin{algorithm}[htb!]
\SetAlgoLined
    \textbf{Input:} $\mathbf{b} \in [0,1]^n$
    
    \SetKwFunction{FMain}{$\mathbf{ZV}$}
    \SetKwProg{Fn}{Function}{:}{}
    \Fn{\FMain{$\mathbf{b}$  }}
    {
    
    $\mathcal{A} \gets \{\mathbf{b}\}$
    
    $\mathcal{B} \gets \{\mathbf{b}\}$ 
    
    \While{$|\mathcal{B}| \neq 0$}{

    
        $\mathcal{B} \gets \emptyset$
    
        \For{$\mathbf{a} \in \mathcal{A}$}{
        
            \For{$i \in \{1, \ldots, |\mathbf{a}|\}$}{

                \tcc{Check for those components that are equal $1/2$}
                \uIf{$a[i] = 1/2$}{ 
                
                    $\mathbf{b} \gets \mathbf{a}$
                    
                    $b[i] \gets 1$
                    
                    $\mathcal{B} \gets \mathcal{B} \cup \mathbf{b}$ \tcc{Insert the first copy into $\mathcal{B}$}
                    
                    $\mathbf{b} \gets \mathbf{a}$
                    
                    $b[i] \gets 0$
                    
                    $\mathcal{B} \gets \mathcal{B} \cup \mathbf{b}$ \tcc{Insert the second copy into $\mathcal{B}$}
                    
                    \textbf{break}
                
                }
                
            }
        
        }
        
        \uIf{$|\mathcal{B}| \neq 0$}{
        
            $\mathcal{A} \gets \mathcal{B}$
        
        }
    
    }
    
    \textbf{return} $\mathcal{A}$ \tcc{return a set of vectors}

    }
\caption{Zeros Verification (ZV)}
\label{alg:zeros_verification}
\end{algorithm}

After checking for null values in $\mathbf{v}_i$, the sets in which it must be placed must be defined. It should be noted that at the end of this process, $\mathbf{b}$ is a binary vector that represents the index in which $\mathbf{v}_i$ must be placed (there might be more than one $\mathbf{b}$ for a single vertex). Equation \eqref{eq:decode_binary} presents the process for converting $\mathbf{b}$ into a decimal number $q$. The calculation of the index is presented in Algorithm \ref{alg:array_position}.
\begin{equation}
    q = \sum_{i=1}^n b_i 2^{i-1}
    \label{eq:decode_binary}
\end{equation}

\begin{algorithm}[htb!]
\SetAlgoLined
    \textbf{Input:} $\mathbf{b} \in \{0,1\}^m$
    
    \SetKwFunction{FMain}{$\mathbf{AP}$}
    \SetKwProg{Fn}{Function}{:}{}
    \Fn{\FMain{$\mathbf{b}$  }}
    {
    
    let $q$ be an integer
    
    $q \gets 0$
    
    \For{$i \in \{1, \ldots, n\}$}{
    
        $q \gets q + b[i] 2^{i-1}$
    
    }
    
    \textbf{return} $q$

    }
\caption{Get Array Position (AP)}
\label{alg:array_position}
\end{algorithm}

Figure \ref{fig:vertex_separation_process} illustrates a toy example of how this process works for a single vertex $\mathbf{v}_i$. In this example, $\mathbf{v}_i = (-2, 0)$, with the second component being null. As a result, $\mathbf{b} = (0,\frac{1}{2})$.
Next, two copies of $\mathbf{b}$ are created. The components of the first (second) copy that are equal to $\frac{1}{2}$ are replaced by $0$ ($1$).
In this case, as there was only one null component, the generated vectors are $(0,0)$ and $(0,1)$.
The vector generated from $\mathbf{b}$ contains the binary representation of the indices of the sets that $\mathbf{v}_i$ must be placed in. In this example, $\mathbf{v}_i$ must be placed in the orthants with indices $0$ and $2$.
It is important to note that the indices used by the algorithm to identify the destination sets are different from the traditional orthant enumeration (2$^{\tt nd}$ and 3$^{\tt rd}$ quadrants in this case).

\begin{figure}[htb!]
    \centering
    \includegraphics[scale=0.50]{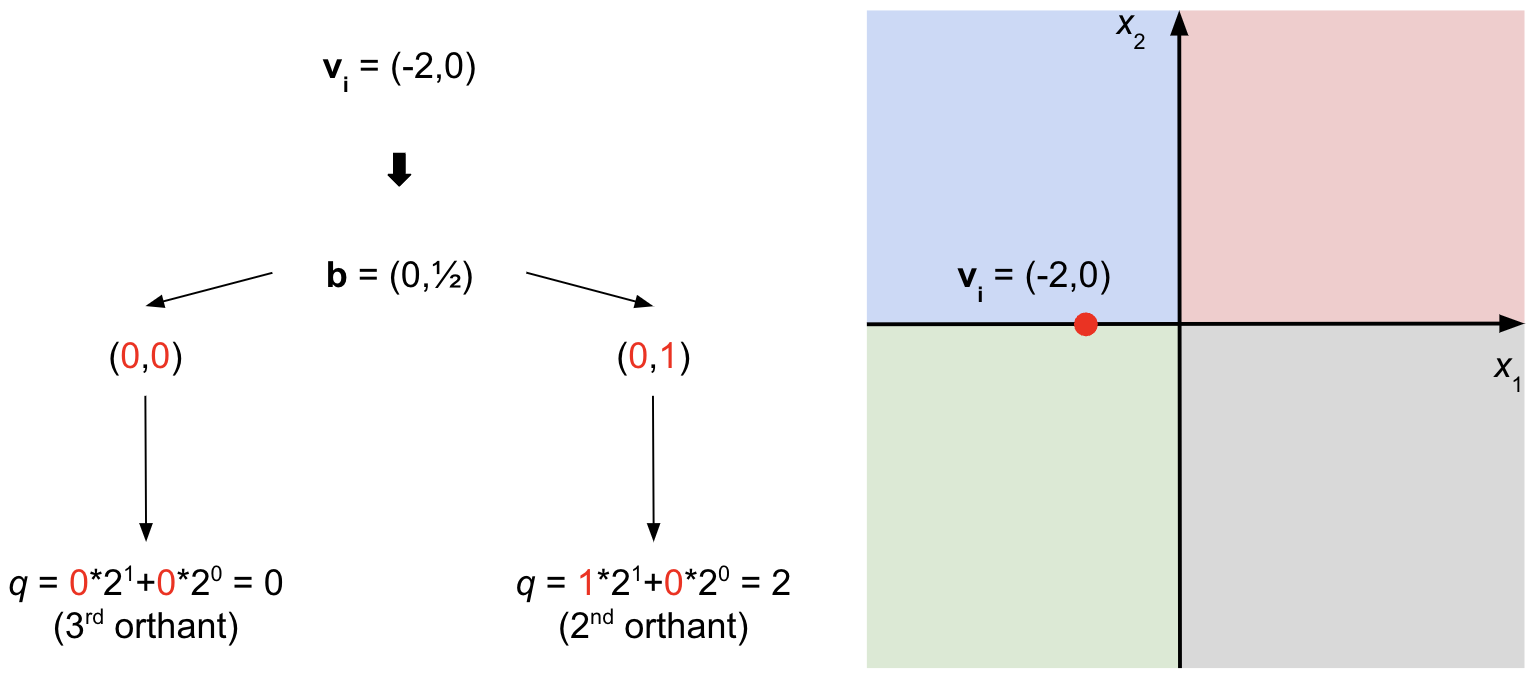}
    \caption{Vertex separation process for a single vertex. Given a vertex $\mathbf{v}_i = (-2, 0)$, the associated $\mathbf{b}$ is given by $(0,\frac{1}{2})$. As we have that there is a component of $\mathbf{b}$ equal to $\frac{1}{2}$, we split it into two vectors in such a way that the component with value equal to $\frac{1}{2}$ is substituted by $1$ and $0$ in the newly created vectors. These new vertices contain the binary representation of the orthant's index that $\mathbf{v}_i$ belong to. In the presented example, $\mathbf{v}_i$ belong to 2 different orthants (the blue and the green ones).}
    \label{fig:vertex_separation_process}
\end{figure}

The orthant separation complete process is computed accordingly by Algorithm \ref{alg:orthant_separation}.

\begin{algorithm}[htb!]
\SetAlgoLined
    \textbf{Input:} $\mathbf{V} \in \mathbb{R}^{o \times n}$, $\mathbf{V} = [\mathbf{v}[1], \ldots, \mathbf{v}[o]]$, $\mathbf{v}[i] \in \mathbb{R}^n \, , \, \forall i \in \{1, \ldots, o\}$
    
    \SetKwFunction{FMain}{$\mathbf{SP}$}
    \SetKwProg{Fn}{Function}{:}{}
    \Fn{\FMain{$\mathbf{V}$  }}
    {
    
    let $\mathbf{Z}[1 \ldots 2^n]$ be an array
    
    \For{$i \in \{1, \ldots, n\}$}{
    
        let $\mathbf{b}[1, \ldots, n]$ be an array
    
        $\mathbf{b} \gets \frac{1}{2}({\tt sign}(\mathbf{v}[i]) + \mathbf{1}_{|\mathbf{v}[i]|})$

        \CommentSty{\% identify if there are null coordinates}

            $\mathcal{B} \gets \textbf{ZV}(\mathbf{b})$ \tcc{Identify the orthants to which $v[i]$ belongs}
            
            \For{$\mathbf{b'} \in \mathcal{B}$}{
                
                $q \gets \textbf{AP}(\mathbf{b'})$  \tcc{Calculate index position of orthant $\mathbf{b}'$ to which $\mathbf{v}[i]$ belongs}
                
                \uIf{$\mathbf{Z}[q] = \emptyset$}{
                
                    $\mathbf{Z}[q] \gets \mathbf{v}[i]$
                
                } \uElse{
                    
                    $\mathbf{Z}[q] \gets \left [\mathbf{Z}[q] \, | \, \mathbf{v}[i]\right ]$ \tcc{Appending the new vertex}
                
                }
            
            }
    
    }
    
    \textbf{return} $\mathbf{Z}$

    }
\caption{Separate points per orthant (SP)}
\label{alg:orthant_separation}
\end{algorithm}

Continuing the example from the previous subsection, the orthant separation performed by Algorithm \ref{alg:orthant_separation} will divide the polytope illustrated in Figure \ref{fig:intersection_identification} into four separate polytopes, one for each orthant, as shown in Figure \ref{fig:orthant_separation}. It is worth noting that vertices $\mathbf{v}_5$, $\mathbf{v}_6$, $\mathbf{v}_7$, $\mathbf{v}_8$, and $\mathbf{v}_9$ have been placed in more than one set, as they are located in more than one orthant.

\begin{figure}[!ht]
    \centering
    \includegraphics[scale=0.3]{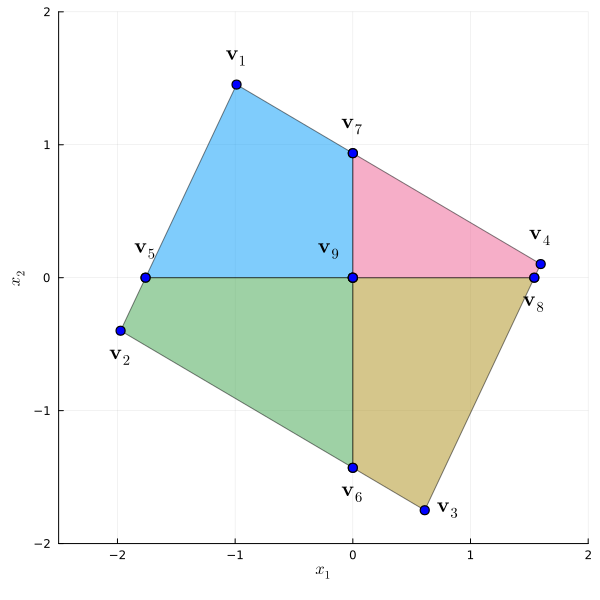}
    \caption{Representation of the separation of the polytope $\mathcal{P}$ after the application of the affine map.}
    \label{fig:orthant_separation}
\end{figure}

\subsubsection*{Exact Polytope Network Mapping}

The layer mapping and the complete process for the Exact Polytope Network Mapping is formalized in Algorithm \ref{alg:layer_map_2}. The algorithm takes as input the vertices of the input polytope aimed to be verified. 
Then, it computes the affine map of these vertices (calculated by the algorithm AM) and identifies the edges between adjacent vertices (with the algorithm EI). 
For the vertices connected by an edge, the algorithm computes the intersection of the corresponding edge with each orthant's supporting hyperplane, when those vertices are not in the same orthant (by means of the algorithm II), and verifies whether or not the origin belongs to the polytope under analysis (checked by the algorithm OS).
Next, it separates all of these points in different sets, where each set contains those vertices that belong to a single orthant (computed by the algorithm SP). 
Finally, the algorithm performs the $ReLU$ mapping and removes non-vertices for each partition generated. 
As later presented, the $ReLU$ mapping preserves the convexity inside a given orthant, though there are points that are not vertices after the application of the non-linear mapping (check vertex $\mathbf{v}_3$ in Figure \ref{fig:relu_map}).

\begin{algorithm}[hbt!]
\SetAlgoLined
    \textbf{Input:} $\mathbf{V} \in \mathbb{R}^{o \times n}$, \tcc{vertices of the polytope} 
    \qquad \quad 
    $\mathbf{W}_l$ and $\boldsymbol{\theta}_l$ for $l \in \{1, \ldots, L\}$ \tcc{network layers' weights and biases}

    \SetKwFunction{FMain}{$\mathbf{EPNM}$}
    \SetKwProg{Fn}{Function}{:}{}
    \Fn{\FMain{$\mathbf{V}, \mathbf{W}, \boldsymbol{\theta}$  }}
    {   
    
    let $\mathcal{P}$ be an empty set
    
    $\mathcal{P} \gets \{\mathbf{V}\}$ 
    \tcc{The list of polytopes's vertices; starting with the input polytope}
    
    \For{$l \in \{1, \ldots, L\}$}{ 
    
        $\hat{\mathcal{P}} \gets \emptyset$ \tcc{$\hat{\mathcal{P}}$ is the auxiliary set that represents $\mathcal{P}$ in the next layer}
        
        \For{$\mathbf{Z} \in \mathcal{P}$}{ 
        \tcc{iterate for each set $\mathcal{Z}$ of vertices in $\mathcal{P}$}
        
            \uIf{$l > 1$}{
            
                $\mathbf{Z} \gets \textbf{ReLU}(\mathbf{Z})$ \tcc{Apply ReLU mapping to the vertex set $\mathbf{Z}$}
                
                $\mathbf{Z} \gets \textbf{RP}(\mathbf{Z})$ \tcc{Remove non-vertex points from vertex set $\mathbf{Z}$}
            
            }
        
            $\mathbf{Z} \gets \textbf{AM}(\mathbf{Z}, \mathbf{W}_l, \boldsymbol{\theta}_l)$ \tcc{Perform affine mapping, $\mathbf{Z} \gets \mathbf{W}_l\mathbf{Z} + \boldsymbol{\theta}_l$}
        
            \uIf{$l < L$}{
            
                $\mathbf{E} \gets \textbf{EI}(\mathbf{Z})$ \tcc{Compute edge-skeleton $\mathbf{E}$ of the polyhedron given by vertices $\mathbf{Z}$}
            
                $\mathbf{Z} \gets \textbf{II}(\mathbf{Z}, \mathbf{E})$ \tcc{Add edge intercepts with orthant's supporting hyperplanes to vertex set}
                
                $\mathbf{Z} \gets \textbf{OS}(\mathbf{Z})$ \tcc{Add origin to the vertex set if needed}
                
                $\mathbf{Z} \gets \textbf{SP}(\mathbf{Z})$ \tcc{Transform the vertex set into a list of vertices for each orthant, considering the orthants to which each vertex belongs}
            
            }
            
            \For{$k \in \{1, \ldots, |\mathbf{Z}|\}$}{
              \uIf{$ \mathbf{Z}[k]\neq \emptyset $}{
                $\hat{\mathcal{P}} \gets \hat{\mathcal{P}} \cup \mathbf{Z}[k]$ 
              }
            }
        
        }
        
        $\mathcal{P} \gets \hat{\mathcal{P}}$
    
    }
 
    \textbf{return} $\mathcal{P}$

    }
\caption{Exact Polytope Network Mapping}
\label{alg:layer_map_2}
\end{algorithm}

Finally, after the application of the $ReLU$ mapping and the removal of non-vertices (RP) for each partition in the set, as illustrated in Figure \ref{fig:orthant_separation}, the result of the EPNM algorithm for a single layer mapping is shown in Figure \ref{fig:relu_map_2}. 
Observe that this output, which is a union of sets, represents the mapping of a unique layer. 
For each of the sets that result from the previous layer, the algorithm computes the exact mapping for the next layer. 
This process will repeat until the algorithm maps all the sets to the last layer of the neural network.
For instance, considering that the visual problem has one extra layer, each of the four output sets (one set for each orthant, as the intersection of $\mathcal{P}$ with each orthant is not empty) will be mapped to the next layer following the same described process.

\begin{figure}[!ht]
    \centering
    \includegraphics[scale=0.3]{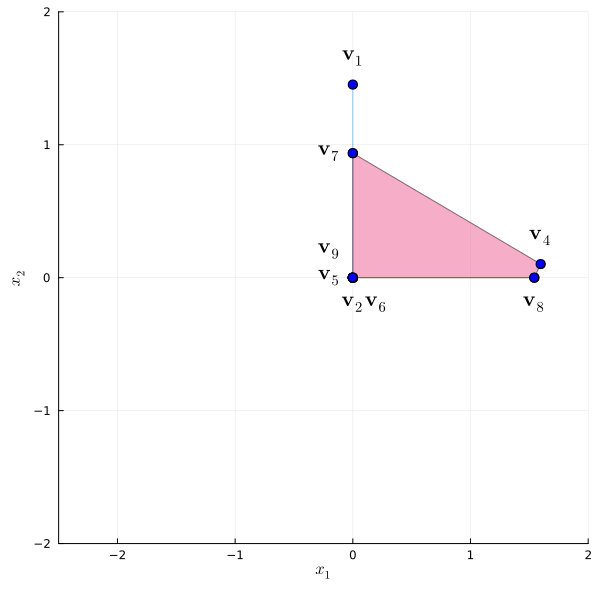}
    \caption{Representation of the output reachable set for a single layer of the EPNM algorithm.}
    \label{fig:relu_map_2}
\end{figure}

\subsection{Partially approximated polytope network mapping (PAPNM)}

We have proposed a third approach, which involves a slight modification of the EPNM algorithm. Instead of an exact mapping between layers, this approach allows for the merging of some of the resulting sets. By computing the convex hull of the union for some of the output sets, we aim to achieve a more accurate approximation of the output set (compared to APNM) while also reducing the execution time of the algorithm (compared to EPNM). The result is presented in Algorithm \ref{alg:layer_map_3}.

\begin{algorithm}
\SetAlgoLined
    \textbf{Input:} $\mathbf{V} \in \mathbb{R}^{o \times n}$, \tcc{vertices of the polytope} 
    \qquad \quad 
    $\mathbf{W}_l$ and $\boldsymbol{\theta}_l$ for $l \in \{1, \ldots, L\}$ \tcc{network layers' weights and biases}
    \qquad \quad 
    \tcc{$d$ is the size of each group of sets that will be merged}

    \SetKwFunction{FMain}{$\mathbf{PAPNM}$}
    \SetKwProg{Fn}{Function}{:}{}
    \Fn{\FMain{$\mathbf{V}, \mathbf{W}, \boldsymbol{\theta}, d$  }}
    {
    
    let $\mathcal{P}$ be an empty set
    
    $\mathcal{P} \gets \{\mathbf{V}\}$
    
    \For{$l \in \{1, \ldots, L\}$}{
    
        $\hat{\mathcal{P}} \gets \emptyset$
        
        \For{$\mathbf{Z} \in \mathcal{P}$}{
        
            \uIf{$l > 1$}{
            
                $\mathbf{Z} \gets \textbf{ReLU}(\mathbf{Z})$
                
                $\mathbf{Z} \gets \textbf{RP}(\mathbf{Z})$
            
            }
        
            $\mathbf{Z} \gets \textbf{AM}(\mathbf{Z}, \mathbf{W}_l, \boldsymbol{\theta}_l)$
        
            \uIf{$l < L$}{
            
                $\mathbf{E} \gets \textbf{EI}(\mathbf{Z})$
            
                $\mathbf{Z} \gets \textbf{II}(\mathbf{Z}, \mathbf{E})$
                
                $\mathbf{Z} \gets \textbf{OS}(\mathbf{Z})$
                
                $\mathbf{Z} \gets \textbf{SP}(\mathbf{Z})$
                
                $\mathbf{Z} \gets \textbf{MS}(\mathbf{Z}, d)$ \tcc{new function for merging sets}
            
            }
            
            \For{$k \in \{1, \ldots, |\mathbf{Z}|\}$}{
            
                $\hat{\mathcal{P}} \gets \hat{\mathcal{P}} \cup \mathbf{Z}_k$
            
            }
        
        }
        
        $\mathcal{P} \gets \hat{\mathcal{P}}$
    
    }
 
    \textbf{return} $\mathcal{P}$

    }
\caption{Partially approximate polytope network mapping}
\label{alg:layer_map_3}
\end{algorithm}

The main difference between Algorithm \ref{alg:layer_map_2} and Algorithm \ref{alg:layer_map_3} is the application of the $MS$ procedure after the separation process. 
The $MS$ algorithm comprises the merging procedure, where a set of sets of vertices is given as input. 
As each set of vertices represents a single polytope, this procedure merges some of these sets of vertices to produce fewer sets compared to the exact mapping. 
It is worth noting that the output of $MS$ is an overapproximation of the input set it receives, and that the extreme case in which a single set of vertices is computed is exactly the case of APNM.

Here, we propose a basic approach for the merging process in Algorithm \ref{alg:merge_sets}. 
The $MS$ algorithm takes the set $\mathcal{P}$ which contains the sets of vertices and an integer $d$ as inputs. 
The sets of vertices are then grouped into sets of size $d$ and each group is merged (i.e, for the case where $\mathcal{P}$ has 13 sets and $d=3$, there will be 4 groups of 3 sets and 1 group of 1 set). 
It is important to note that this is just a simple example of the merging procedure and that different heuristics can be implemented to improve this process.

\begin{algorithm}[!ht]
\SetAlgoLined
    \textbf{Input:} 
    $\mathcal{P} = \{\mathbf{V}_1, \ldots, \mathbf{V}_q\}, \mathbf{V}_i \in \mathbb{R}^{o_i \times n}, i \in \{1, \ldots, q\}$ \tcc{set of polytopes' vertices}
    
    \qquad \quad
    $d \in \mathbb{N}$ \tcc{$d$ is the size of each group of sets that will be merged}

    \SetKwFunction{FMain}{$\mathbf{MS}$}
    \SetKwProg{Fn}{Function}{:}{}
    \Fn{\FMain{$\mathcal{P}, d$  }}
    {
    
    let $\mathcal{A}$ be an empty set
    
    \For{$i \in \{1, \ldots, \left \lceil \frac{|\mathcal{P}|}{d} \right \rceil \}$}{
    
        let $\mathbf{B}[1 \ldots n]$ be an array
        
        \For{$j \in \{(i-1)\times d+1, \ldots, \min(i\times d, |\mathcal{P}|) \}$}{
        
            \uIf{$\mathbf{B} = \emptyset$}{
                
                $\mathbf{B} \gets \mathbf{V}[j]$
                
            } \uElse{
                
                $\mathbf{B} \gets (\mathbf{B} | \mathbf{V}[j])$
                
            }
        
        }
        
        $\mathcal{A} \gets \mathcal{A} \cup \mathbf{B}$
    
    }
    
    \textbf{return} $\mathcal{A}$

    }
\caption{Merge Sets (MS)}
\label{alg:merge_sets}
\end{algorithm}

The use of this merging procedure allows for different levels of approximation, resulting in a range of possible approximations from the exact mapping (EPNM) to the coarser case (APNM). It is important to ensure that the merging procedure returns an overapproximation of the exact mapping of each layer, which is a necessary condition to ensure the soundness of PAPNM.

\section{Demonstrations}
\label{sec:demonstrations}

We present in this section demonstrations that provide theoretical guarantees for the correctness of each proposed algorithm.

\subsection{Identification of adjacent vertices}

Let $\mathcal{P} = \left \{\sum_{i=1}^o \lambda_i \mathbf{v}_i \mid \sum_{i=1}^o \lambda_i = 1, \lambda_i \geq 0 \text{ e } \mathbf{v}_i \in \mathcal{V}, \forall i \in \{1, \ldots, o\} \right\}$ be a convex V-polytope defined as the convex combination of its vertices, where $\mathcal{V} = \{\mathbf{v}_1, \ldots, \mathbf{v}_o\}$ is the set of vertices of a polyhedron $\mathcal{P}$.

\begin{deff} Given a vector $\mathbf{c} \in \mathbb{R}^n$  and $\delta = \max \{\mathbf{c}^T \mathbf{x} \mid \mathbf{x} \in \mathcal{P}\}$, we have that $\mathcal{H} = \{\mathbf{x} \mid \mathbf{c}^T \mathbf{x} = \delta\}$ is a supporting hyperplane of $\mathcal{P}$. 
\end{deff}

\begin{deff} $\mathcal{F}$ is a face of $\mathcal{P}$ if $\mathcal{F} = \mathcal{P}$ or $\mathcal{F} = \mathcal{P} \cap \mathcal{H}$ for some supporting hyperplane $\mathcal{H}$. In other words, $\mathcal{F}$ is a face of $\mathcal{P}$ if, and only if, $\mathcal{F}$ is the set of optimal solutions for $\max  \{\mathbf{c}^T \mathbf{x} \mid \mathbf{x} \in \mathcal{P}\}$ for a given $\mathbf{c} \in \mathbb{R}^n$ \cite{schrijver2003combinatorial}.
\end{deff}

\begin{deff} $\mathbf{v}_i$ e $\mathbf{v}_j$ are adjacent vertices if there is a vector $\mathbf{c} \in \mathbb{R}^n$ such that $\mathbf{c}^T \mathbf{v}_i = \mathbf{c}^T \mathbf{v}_j = \max \{\mathbf{c}^T \mathbf{x} \mid \mathbf{x} \in \mathcal{P}\} > \mathbf{c}^T \mathbf{v}_k, \forall k \in \{1, \ldots, o\}, k\neq i \text{ and } k \neq j$ \cite{schrijver2003combinatorial}. 
\label{def:vertices_adj}
\end{deff}

Notice that the face $\mathcal{F}$ is an edge in the particular case stated by Definition \ref{def:vertices_adj}. We propose that:

\begin{proposition} Two extreme point $\mathbf{v}_i, \mathbf{v}_j \in \mathcal{V}$ are adjacent if, and only if, it is not possible to compute the median point, $\Bar{\mathbf{v}} = (\mathbf{v}_i + \mathbf{v}_j)/2$, as a convex combination of the vertices in $\mathcal{V}\setminus\{\mathbf{v}_i\}$ and $\mathcal{V}\setminus\{\mathbf{v}_j\}$.
\label{prop:adjacent_vertices}
\end{proposition}

\begin{proof}[\textbf{Proof}]
Given two vertices $\mathbf{v}_i$ and $\mathbf{v}_j$, we have two possibilities regarding their adjacency: 1) they are adjacent, or else 2) they are not adjacent.

For the first case, as $\mathbf{v}_i$ and $\mathbf{v}_j$ are adjacent, we have by definition that there is a hyperplane $\mathcal{H}_a = \{\mathbf{x} \mid \mathbf{c}^T \mathbf{x} = d \}$ that contains both $\mathbf{v}_i$ and $\mathbf{v}_j$, such that $\mathcal{H}_a \cap \mathcal{P} = \max \{\mathbf{c}^T \mathbf{x} \mid \mathbf{x} \in \mathcal{P}\} > \mathbf{c}^T \mathbf{v}_k, \forall k \in \{1, \ldots, o\}, k\neq i \text{ and } k \neq j$ for a given $\mathbf{c} \in \mathbb{R}^n$.
Therefore, all the remaining extreme points from $\mathcal{P}$, except from $\mathbf{v}_i$ and $\mathbf{v}_j$, are in the open half-space $\mathcal{H}_b = \{\mathbf{x} \mid \mathbf{c}^T \mathbf{x} < d \}$. Consequently, the median point of $\mathbf{v}_i$ and $\mathbf{v}_j$, denoted by $\Bar{\mathbf{v}} = {(\mathbf{v}_i + \mathbf{v}_j)}/{2}$, can not be expressed as a convex combination of $\mathbf{v}_k$, for $k \in \{1, \ldots, o\} \setminus \{i\}$ or $k \in \{1, \ldots, o\} \setminus \{j\}$.

Considering that $\mathbf{v}_i$ and $\mathbf{v}_j$ are not adjacent, let $\mathcal{P}_i$ and $\mathcal{P}_j$ be two polytopes given by the convex combination of the extreme points $\mathcal{V} \setminus \{\mathbf{v}_i\}$ and $\mathcal{V} \setminus \{\mathbf{v}_j\}$, respectively.
Hence, there are two distinct possibilities: $\Bar{\mathbf{v}} \in \mathcal{P}_j$ or $\Bar{\mathbf{v}} \not \in \mathcal{P}_j$. For the first case, as $\Bar{\mathbf{v}} \in \mathcal{P}_j$, then we can compute this point as a convex combination of the vertices in $\mathcal{V} \setminus \{\mathbf{v}_j\}$.
For the second case, as $\Bar{\mathbf{v}} \not \in \mathcal{P}_j$ and $\Bar{\mathbf{v}} \in \mathcal{P}$, then $\Bar{\mathbf{v}} \in \mathcal{P} \setminus \mathcal{P}_j$. However, since the vertices that are adjacent to $\mathbf{v}_j$ are in $\mathcal{V} \setminus \{\mathbf{v}_i, \mathbf{v}_j\}$, then $\mathcal{P} \setminus \mathcal{P}_j \subseteq \mathcal{P}_i$. From this fact it follows that $\Bar{\mathbf{v}} \in \mathcal{P}_i$ and, consequently, $\Bar{\mathbf{v}}$ can be expressed as a convex combination of the vertices $\mathcal{V} \setminus \{\mathbf{v}_i\}$.
\end{proof}

\subsection{ReLU convexity inside an orthant}

Let $f:\mathbb{R}^n\rightarrow\mathbb{R}^n$ be a function that denotes the $ReLU$ mapping, defined by the Equation \eqref{eq:ReLUMap}:
\begin{equation}
    f(\mathbf{x})_i = \max(0,x_i)
    \label{eq:ReLUMap}
\end{equation}

\begin{proposition} Given $\mathbf{x},\mathbf{y} \in \mathbb{R}^n$, if $\sup\{\gamma \mid \gamma = {\tt sign}(\mathbf{x})_i - {\tt sign}(\mathbf{y})_i, \forall i \in \{1, \ldots, n\} \} \leq 1$, or, in other words, if $\mathbf{x}$ and $\mathbf{y}$ belong to the same orthant, we have that $f(\mathbf{x}+\mathbf{y}) = f(\mathbf{x}) + f(\mathbf{y})$ and that $f(\alpha\mathbf{x}) = \alpha f(\mathbf{x})$.\textbf{}
\label{prop:relu_convexa_por_ortant}
\end{proposition}

\begin{proof}[\textbf{Proof}] Denoting each orthant of $\mathbb{R}^n$ as $\mathcal{O}_j, \forall j \in \{1, \ldots, 2^n\}$, we can rewrite the $ReLU$ mapping, given by $f$, as $f_j: \mathcal{O}_j \rightarrow \mathbb{R}^n$, such that:
\begin{equation}
    f_j(\mathbf{x}) = \boldsymbol{\Lambda}_j \mathbf{x}
\end{equation}
where $\boldsymbol{\Lambda}_j \in \mathbb{R}^{n \times n}$ is the matrix in which the elements of the principal diagonal associated with a negative component of $\mathbf{x} \in \mathcal{O}_j$ are equal $0$, while the remaining elements are equal $1$ ($\lambda_{k,l}$ for $k=l$). The  elements off the principal diagonal are all equal to zero ($\lambda_{k,l} = 0$ for all $k \neq l$).

Therefore, $f_j$ preserves convexity since it consists of a linear mapping that satisfies both properties stated in the proposition, the addition property:
\begin{align*}
    f_j(\mathbf{x} + \mathbf{y}) &=  \boldsymbol{\Lambda}_j(\mathbf{x} + \mathbf{y})\nonumber \\ 
    &=  \boldsymbol{\Lambda}_j\mathbf{x} + \boldsymbol{\Lambda}_j\mathbf{y}\nonumber \\
    &=  f_j(\mathbf{x}) + f_j(\mathbf{y}) \nonumber
\end{align*}
and the product by a scalar:
\begin{align*}
    f_j(\alpha \mathbf{x}) &=  \boldsymbol{\Lambda}_j(\alpha \mathbf{x}) \nonumber \\ 
      &=  \alpha \boldsymbol{\Lambda}_j \mathbf{x}  \nonumber \\
     &=  \alpha f_j(\mathbf{x}) \nonumber
\end{align*}
where $\alpha \in \mathbb{R}$.

Hence, as $f_j$ satisfies the linearity conditions within each orthant $\mathcal{O}_j$ and the linear mapping preserves convexity, it follows that:
\begin{align*}
    f_j(\theta \mathbf{x} + (1 - \theta) \mathbf{y}) &=  \boldsymbol{\Lambda}_j(\theta \mathbf{x} + (1 - \theta) \mathbf{y}) \nonumber \\ 
    &=  \theta \boldsymbol{\Lambda}_j \mathbf{x} + (1 - \theta) \boldsymbol{\Lambda}_j \mathbf{y} \nonumber \\
    &=  \theta f_j(\mathbf{x}) + (1 - \theta) f_j(\mathbf{y}) \nonumber
\end{align*}
for all $\mathbf{x}, \mathbf{y} \in \mathcal{O}_j$ e $\theta \in [0,1]$. Therefore, the $ReLU$ mapping preserves the convexity inside a given orthant.
\end{proof}

\subsection{V-polytope and half-space intersection}

Let $\mathcal{P} = \left\{\sum_{i=1}^o \lambda_i \mathbf{v}_i \mid \sum_{i=1}^o \lambda_i = 1, \lambda_i \geq 0 \text{ and } \mathbf{v}_i \in \mathcal{V}, \forall i \in \{1, \ldots, o\} \right\}$ be a closed convex polyhedron defined in terms of the convex combination of its vertices (or extreme points), where $\mathcal{V} = \{\mathbf{v}_1, \ldots, \mathbf{v}_n\}$ is the set of the vertices of $\mathcal{P}$. Such a polyhedron can also be defined as a set of inequalities $\mathcal{P} = \{\mathbf{x} \mid \mathbf{C} \mathbf{x} \leq \mathbf{d}\}$, such that $\mathbf{C} \in \mathbb{R}^{m \times n}$ and $\mathbf{d} \in \mathbb{R}^m$.

\begin{figure}[!ht]
     \centering
     \begin{subfigure}[b]{0.3\textwidth}
         \centering
         \includegraphics[width=\textwidth]{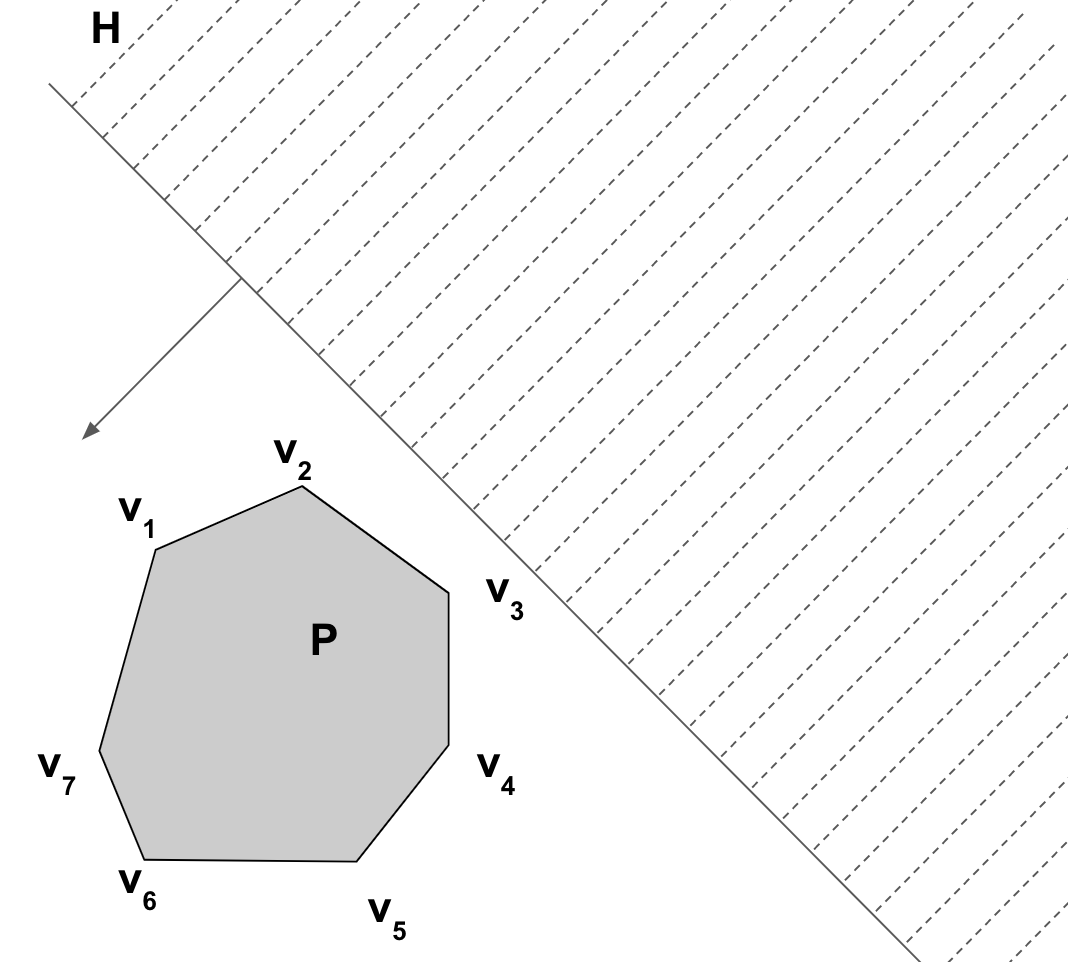}
         \caption{$\mathcal{P} \cap \mathcal{H} = \mathcal{P}$}
         \label{fig:y equals x}
     \end{subfigure}
     \hfill
     \begin{subfigure}[b]{0.3\textwidth}
         \centering
         \includegraphics[width=\textwidth]{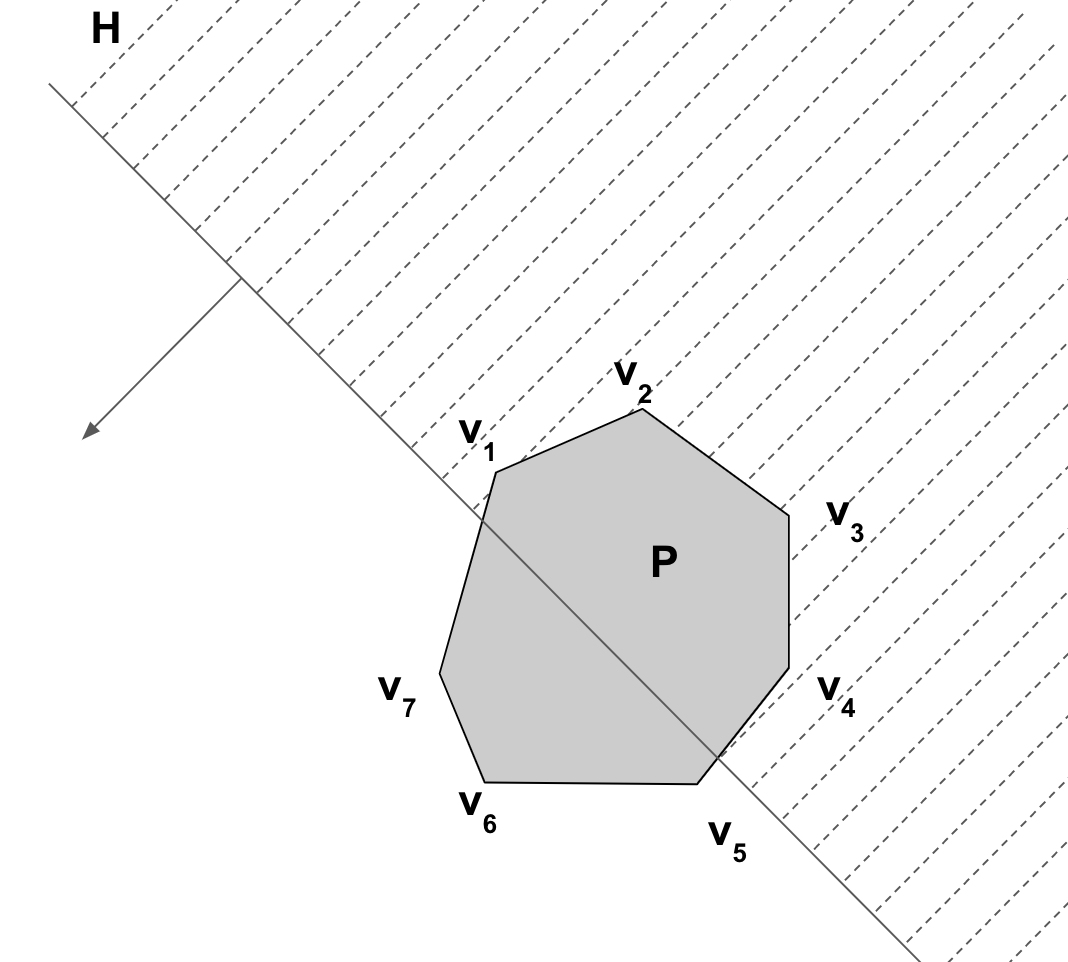}
         \caption{$\mathcal{P} \cap \mathcal{H} \subset \mathcal{P}$}
         \label{fig:three sin x}
     \end{subfigure}
     \hfill
     \begin{subfigure}[b]{0.3\textwidth}
         \centering
         \includegraphics[width=\textwidth]{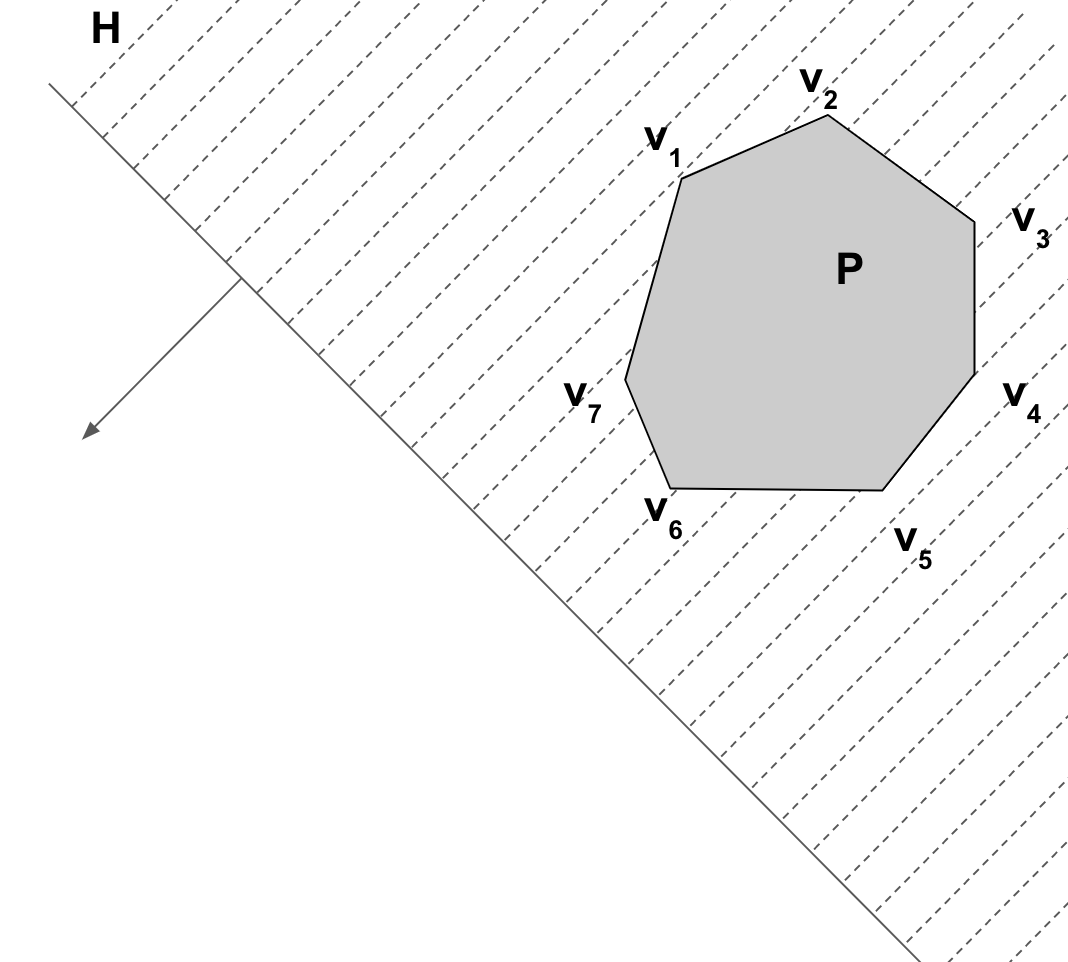}
         \caption{$\mathcal{P} \cap \mathcal{H} = \emptyset$}
         \label{fig:five over x}
     \end{subfigure}
        \caption{Representation of all the three different possible cases with respect to the intersection of $\mathcal{P}$ with a half-space $\mathcal{H}$.}
        \label{fig:three_possibilities}
\end{figure}

There are three different possible cases regarding the intersection of $\mathcal{P}$ with a half-space $\mathcal{H} = \{\mathbf{x} \mid \mathbf{a}^T \mathbf{x} \leq b\}$, where $\mathbf{a} \in \mathbb{R}^n$ and $b \in \mathbb{R}$ (Figure \ref{fig:three_possibilities} presents a visual representation for each case):
\begin{enumerate}
    \item $\mathcal{P} \cap \mathcal{H} = \mathcal{P}$;
    \item $\mathcal{P} \cap \mathcal{H} \subset \mathcal{P}$;
    \item $\mathcal{P} \cap \mathcal{H} = \emptyset$.
\end{enumerate}

The first and the third cases are trivial. For the first case, we have that all the extreme points of $\mathcal{P}$ are in $\mathcal{H}$. For the third case none of the extreme points of $\mathcal{P}$ belongs to $\mathcal{H}$.

For the second case, a new face is generated for $\mathcal{P}$, given by $\mathcal{P}_s = \mathcal{P} \cap \mathcal{H}_s$, where $\mathcal{H}_s = \{\mathbf{x} \mid \mathbf{a}^T \mathbf{x} = b\}$ such that $\mathbf{a} \in \mathbb{R}^n$ and $b \in \mathbb{R}$. Note that $\mathcal{H}_s$ is the supporting hyperplane of $\mathcal{H}$. Therefore, those vertices of $\mathcal{P}$ that are not in $\mathcal{H}$, are not extreme points of $\mathcal{P} \cap \mathcal{H}$. Thus, it is necessary to find those extreme points of the new polytope $\mathcal{P}_h = \mathcal{P} \cap \mathcal{H}$. Figure \ref{fig:structure_vizualization} presents a visualization of the elements previously defined.

\begin{figure}[!ht]
    \centering
    \includegraphics[scale=0.35]{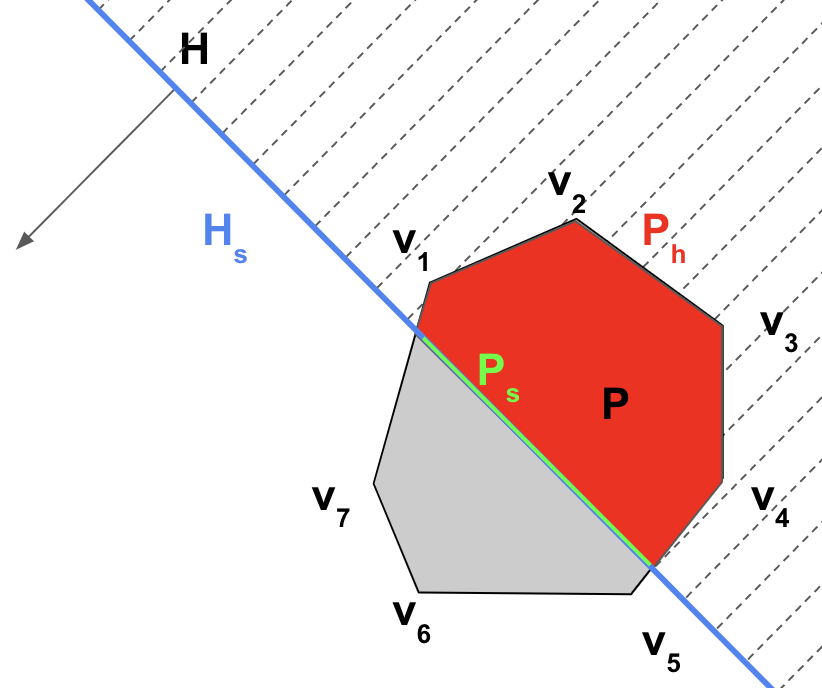}
    \caption{Representation of the elements of interest from the intersection between $\mathcal{P}$ and $\mathcal{H}$ for the second case. The supporting hyperplane $\mathcal{H}_s$ of $\mathcal{H}$ is presented in blue. In green we have the intersection between $\mathcal{H}_s$ and $\mathcal{P}$, and in red the result of the intersection between $\mathcal{P}$ with $\mathcal{H}$.
    Notice that, for this example,  $\mathcal{V}_h = \{\mathbf{v}_1, \mathbf{v}_2, \mathbf{v}_3, \mathbf{v}_4\}$.}
    \label{fig:structure_vizualization}
\end{figure}

We denote by $\mathcal{V}_{h}$ the subset of vertices of $\mathcal{P}$ that belong to $\mathcal{H}$, given by $\mathcal{V}_{h} = \mathcal{V} \cap \mathcal{H}$.
Let $\mathcal{E} = \{\mathcal{E}_1, \ldots, \mathcal{E}_p\}$ be the set of edges of $\mathcal{P}$. $\mathcal{E}_i$ denotes the set of points that belong to the $i$-th edge of $\mathcal{P}$.
Finally, we denote by $\mathcal{V}_p$ the set of vertices from $\mathcal{P}_h$. Observe that $\mathcal{V}_h \subseteq \mathcal{V}_p$.

Let $\mathbf{c}$ be a non-zero vector and $\delta = \max \{\mathbf{c}^T \mathbf{x} \mid \mathbf{C} \mathbf{x} \leq \mathbf{d}\}$. The affine hyperplane $\mathcal{H}_a = \{\mathbf{x} \mid \mathbf{c}^T \mathbf{x} = \delta\}$ is a supporting hyperplane of $\mathcal{P}$.

\begin{deff} A subset $\mathcal{F}$ of $\mathcal{P}$ is called a face of $\mathcal{P}$ if $\mathcal{F} = \mathcal{P}$ or $\mathcal{F} = \mathcal{P} \cap \mathcal{H}_a$ \cite{schrijver1998theory}.
\end{deff}

\begin{deff} $\mathcal{F}$ is a face of $\mathcal{P}$ if, and only if, $\mathcal{F}$ is not empty and $\mathcal{F} = \{\mathbf{x} \in \mathcal{P} \mid \mathbf{C}' \mathbf{x} = \mathbf{d}'\}$, for a subsystem $\mathbf{C}' \mathbf{x} \leq \mathbf{d}'$ of $\mathbf{C} \mathbf{x} \leq \mathbf{d}$ \cite{schrijver1998theory}.
\end{deff}

\begin{proposition} Considering that $\mathcal{P} \cap \mathcal{H} \subset \mathcal{P}$ and given that $\mathcal{V}_p$ is the set of extreme points of $\mathcal{P}_h$, if $\mathbf{v} \in \mathcal{V}_p$ and $\mathbf{v} \not \in \mathcal{V}_h$ then $\mathbf{v} \in \mathcal{P}_s$.
\label{prop:intersection_part1}
\end{proposition}

\begin{proof}[\textbf{Proof}] By definition, an extreme point of a polytope is a $0$-dimensional face, thus:
\begin{equation}
    \mathcal{F} = \{\mathbf{x} \mid \mathbf{C}'\mathbf{x} = \mathbf{d}'\}
\end{equation}
where $\mathbf{C}'\mathbf{x} \leq \mathbf{d}'$ is a subsystem of $\mathbf{C} \mathbf{x} \leq \mathbf{d}$. Remember that $\mathcal{P} = \{\mathbf{x} \mid \mathbf{C} \mathbf{x} \leq \mathbf{d}\}$ and that $\mathcal{H} = \{\mathbf{x} \mid \mathbf{a}^T \mathbf{x} \leq b\}$.
Furthermore, since a vertex $\mathbf{v}$ is the particular case of a face given by the intersection of $n$ hyperplanes, for a $n$-dimensional space, then:
\begin{equation}
    \mathbf{v} = \{\mathbf{x} \mid \mathbf{C}''\mathbf{x} = \mathbf{d}'' \}
\end{equation}
for some $\mathbf{C}'' \in \mathbb{R}^{n \times n}$ and $\mathbf{d}'' \in \mathbb{R}^n$, such that ${\tt det}(\mathbf{C}'') \neq 0$ and $\mathbf{v} \neq \emptyset$.

However, for $\mathbf{x} \in \mathcal{H} \setminus \mathcal{H}_s$, no new solution for the subsystem $\mathbf{C}''\mathbf{x} = \mathbf{d}''$ subsystem of:
\begin{equation}
    \begin{bmatrix}
    \mathbf{C} \\
    \mathbf{a}^T
    \end{bmatrix} \mathbf{x} \leq 
    \begin{bmatrix}
    \mathbf{d} \\
    b
    \end{bmatrix}
\end{equation}
such that ${\tt det}(\mathbf{C}'') \neq 0$ and $\mathbf{v} \neq \emptyset$, is possible. Put in other words, as no constraint was placed in $\mathcal{H} \setminus \mathcal{H}_s$, no new vertex was generated in $\mathcal{P} \cap ( \mathcal{H} \setminus \mathcal{H}_s )$.
Consequently, those vertices generated by the intersection with $\mathcal{H}$, in case they exist, must belong to $\mathcal{H}_s$. Hence, if $\mathbf{v} \in \mathcal{V}_p$ and $\mathbf{v} \not \in \mathcal{V}_h$, then $\mathbf{v} \in \mathcal{P}_s$.
\end{proof}

\begin{proposition} Considering that $\mathcal{P} \cap \mathcal{H} \subset \mathcal{P}$, if $\mathcal{V}_p$ is the set of extreme points of $\mathcal{P}_h$, then $\mathcal{V}_p\setminus \mathcal{V}_h = \mathcal{E} \cap \mathcal{H}_s$.
\label{prop:intersection_part2}
\end{proposition}

\begin{proof}[\textbf{Proof}] As $\mathcal{P} \cap \mathcal{H} \subset \mathcal{P}$, then $\mathcal{P}_s$ is a face of $\mathcal{P}_h$. Furthermore, if $\mathcal{F}$ is a face of $\mathcal{P}_s$, then $\mathcal{F}$ is a face of $\mathcal{P}_h$. Therefore, the 0-dimensional faces of $\mathcal{P}_s$ (vertices) are also faces of $\mathcal{P}_h$.

Recall that a vertex is defined by the intersection of $n$ hyperplanes, for a $n$-dimensional space, that $\mathcal{P} = \{\mathbf{x} \mid \mathbf{C} \mathbf{x} \leq \mathbf{d}\}$ and that $\mathcal{H} = \{\mathbf{x} \mid \mathbf{a}^T \mathbf{x} \leq b\}$. Therewith, there are two possible cases for the vertices of $\mathcal{P}_s$:
\begin{itemize}
    \item $\mathbf{v}$ is given by $\mathbf{C}'\mathbf{x} = \mathbf{d}'$, where $\mathbf{C}'x \leq \mathbf{d}'$ is a subsystem of $\mathbf{C} \mathbf{x} \leq \mathbf{d}$;
    
    \item $\mathbf{v}$ is given by $\mathbf{C}''\mathbf{x} = \mathbf{d}''$, where:
    \begin{equation}
        \mathbf{C}'' = 
        \begin{bmatrix}
        \mathbf{C}' \\
        \mathbf{a}^T
        \end{bmatrix}, \ 
        \mathbf{d} '' = 
        \begin{bmatrix}
        \mathbf{d}' \\
        b
        \end{bmatrix}
    \end{equation}
\end{itemize}

For the first case, we trivially observe that $\mathbf{v} \in \mathcal{V}_h$. For the second case, as $\mathbf{C}'' \in \mathbb{R}^{n \times n}$ then $\mathbf{C}' \mathbf{x} = \mathbf{d}'$ represents the intersection of $n-1$ hyperplanes, resulting in an edge of $\mathcal{P}$, given by $\mathcal{E}_i = \{\mathbf{x} \mid \mathbf{C}'\mathbf{x} = \mathbf{d}'\}$. Hence, for the case where $\mathcal{P} \cap \mathcal{H} \subset \mathcal{P}$, if $\mathcal{V}_p$ is the set of vertices of $\mathcal{P}_h$, then $\mathcal{V}_p \setminus \mathcal{V}_h = \mathcal{E} \cap \mathcal{H}_s$.
\end{proof}

\subsection{Origin regarding polytope intersection}

As presented in the previous section, considering that $\mathcal{P}$ is a closed polytope defined by the convex combination of its vertices, the intersection between $\mathcal{P}$ and a half-space $\mathcal{H}$ is given by the convex combination of the vertices from $\mathcal{P}$ in the intersection with $\mathcal{H}$, along with the vertices obtained by the intersection of the supporting hyperplane of $\mathcal{H}$ with the edges of $\mathcal{P}$.

Consequently, the intersection between $\mathcal{P}$ and $\mathcal{O}_j$, where $\mathcal{O}_j$ represents one of the $2^n$ orthants in a $n$-dimensional space, can be rewrite as $\mathcal{P} \cap \mathcal{H}_1 \cap \cdots \cap \mathcal{H}_n$. 
Note that $\mathcal{O}_j = \mathcal{H}_1 \cap \cdots \cap \mathcal{H}_n$ for suitable half-spaces $\mathcal{H}_1,\dots,\mathcal{H}_n$.
 By induction, it can be shown that the vertices of $\mathcal{P} \cap \mathcal{O}_j$ consist of the union of vertices of $\mathcal{P}$ that also belong to $\mathcal{O}_j$ with the vertices given by the intersection of the edges of $\mathcal{P}$ with the supporting hyperplanes of $\mathcal{O}_j$ (denoted as $\mathcal{H}_i, \forall i \in \{1, \ldots, n\}$). We define $\mathcal{O}_j$ as:
\begin{equation}
    \mathcal{O}_j = \{\mathbf{x} \mid \boldsymbol{\Phi}_j \mathbf{x} \leq \mathbf{0}\}
\end{equation}
where $\mathbf{\Phi}_j$ is the suitable matrix given by:
\begin{equation}
    \mathbf{\Phi}_j = 
    \begin{bmatrix}
    \phi_{j,1,1} & 0 & \ldots & 0 & \\
    0 & \phi_{j,2,2} & \ldots & 0 & \\
    \vdots & \vdots & \ddots & \vdots & \\
    0 & 0 & \ldots & \phi_{j,n,n} &
    \end{bmatrix}
\end{equation}
such that $\phi_{j,i,i} \in \{-1, 1\}, \forall i \in \{1, \ldots, n\}$.

We can see that the equation system $\mathbf{\Phi}_j \mathbf{x} = \mathbf{0}$ has the origin as its sole solution for any orthant $j$. Therefore, the unique extreme point of $\mathcal{O}_j$ is the origin for all $j \in \{1, \ldots, 2^n\}$.

\begin{proposition} If the origin belongs to $\mathcal{P}$, then it is an extreme point of $\mathcal{P} \cap \mathcal{O}_j$, for all $j \in \{1, \ldots, 2^n\}$.
\label{prop:pertinencia_origem}
\end{proposition}

\begin{proof}[\textbf{Proof}] Let $\mathcal{P}$ be a convex closed polytope and $\mathcal{O}_j$ a convex cone. It is known that the intersection between $\mathcal{P}$ and $\mathcal{O}_j$ is a closed convex polytope. We have that the vertices from $\mathcal{P}$ that belong to $\mathcal{O}_j$ are also vertices of $\mathcal{P} \cap \mathcal{O}_j$, denoted by $\mathcal{V}_{\mathcal{O}_j}$. Also, we have that $\mathcal{P} \cap \mathcal{O}_j = \mathcal{P} \cap \mathcal{H}_1 \cap \cdots \cap \mathcal{H}_n$.

Based on Proposition \ref{prop:intersection_part2}, we have by induction that the intersection of the edges of $\mathcal{P}$ with the supporting hyperplanes of $\mathcal{O}_j$ are also vertices of $\mathcal{P} \cap \mathcal{O}_j$, denoted by $\mathcal{V}_h$.

For the origin, which is the single vertex of $\mathcal{O}_j$, there are two possible cases:

\begin{enumerate}
    \item the origin is not in $\mathcal{P}$: this is the trivial case in which the origin can not be a vertex of $\mathcal{P} \cap \mathcal{O}_j$, as it is not in $\mathcal{P}$;
    
    \item the origin is in $\mathcal{P}$: in this case, by contradiction we suppose that the origin is not a vertex of $\mathcal{P} \cap \mathcal{O}_j$, denoted by $\mathbf{0}$. Then, there must exist two points $\mathbf{v}, \mathbf{u} \in \mathcal{P} \cap \mathcal{O}_j$, such that:
\begin{equation}
    \mathbf{0} = \lambda \mathbf{v} + (1 - \lambda) \mathbf{u}
\end{equation}
given that $0 < \lambda < 1$, since $\mathbf{0} \neq \mathbf{v}$ and $\mathbf{0} \neq \mathbf{u}$.
As the origin is given by $\mathbf{0} = (0,0, \ldots ,0,0)$, there must exist a solution for the equation $0 = \lambda v_i + (1 - \lambda) u_i$ for all $i \in \{1,\ldots,n\}$, such that:
\begin{equation}
    v_i = \frac{(\lambda - 1) u_i}{\lambda} 
\end{equation}

Since $\frac{(\lambda - 1)}{\lambda} < 0$,  the sign of $v_i$ and $u_i$ must by different if $v_i \neq 0$ and $u_i \neq 0$. However, inside a given orthant there must not exist a point with components that have opposite signs. Hence, by contradiction, if the origin is in $\mathcal{P}$, it must be a vertex of $\mathcal{P} \cap \mathcal{O}_j$.
\end{enumerate} 
\end{proof}

\subsection{Correctness of APNM algorithm}

Let $F:\mathbb{R}^{n} \rightarrow \mathbb{R}^{m}$ denote a neural network mapping, where $L$ is the number of layers, the weights of the layer $l$ are given by $\mathbf{W}_l \in \mathbb{R}^{|\mathbf{x}_l|\times|\mathbf{x}_{l-1}|}$ and the biases by $\boldsymbol{\theta}_l \in \mathbb{R}^{|\mathbf{x}_l|}$. Thereby,  $F(\mathbf{x}) = (F_{L} \circ F_{L-1} \circ \cdots \circ F_1)(\mathbf{x})$ where $F_l(\mathbf{x}) = ReLU(\mathbf{W}_l \mathbf{x} + \boldsymbol{\theta}_l)$ is the mapping for each layer $l \in \{1, \ldots, L-1\}$. The only difference for the last layer is the activation (usually sigmoid for binary problems, or softmax for multi-class problems).

For a given layer $l$, we have that the set $\mathcal{X}_l$ denotes the inputs that we aim to map regarding $F_l$, such that $\mathcal{X}_l = \{\sum_{i=1}^o \lambda_i \mathbf{v_i} \mid \sum_{i=1}^o \lambda_i \land \lambda_i \geq 0 \land \mathbf{v}_i \in \mathcal{V}_l, \forall i \in \{1, \ldots, o\}\}$ and $\mathcal{V}_l$ is the set of vertices of $\mathcal{X}_l$.
   Associated with $\mathcal{X}_l$, there is $\mathcal{Y}_l = \{F_l(\mathbf{x}_l) \mid \mathbf{x}_l \in \mathcal{X}_l \}$, which corresponds to the output set for layer $l$.

Finally, let $\widehat{F}_l$ denote the output mapping for a given layer by the Algorithm \ref{alg:layer_map}, given by:
\begin{equation}
    \widehat{F}_l(\mathcal{V}_l, \mathbf{W}_l, \boldsymbol{\theta}_l) = RP\left(ReLU\bigl(II(AM(\mathcal{V}_l, \mathbf{W}_l, \boldsymbol{\theta}_l),EI(AM(\mathcal{V}_l, \mathbf{W}_l, \boldsymbol{\theta}_l)))\bigr)\right)
    \label{eq:convex_layer_mapping}
\end{equation}
where $AM$ corresponds the affine map given by Algorithm \ref{alg:affine_map}, 
    $EI$ is the edge identification given by the Algorithm \ref{alg:edge_identification}, 
$II$ is the intersection identification defined by Algorithm \ref{alg:orthant_intersection}, 
    $ReLU$ is given by Algorithm \ref{alg:relu_map}, and 
$RP$ stands for removing non-vertices defined by Algorithm \ref{alg:removing_internal_points}. 
   We denote the convex hull mapping for a given set of vertices by $CH$.

\begin{proposition} Given a closed convex polytope $\mathcal{X}_l$ as input set and $\mathcal{V}_l$ as the set of its vertices, then it implies that $\mathcal{Y}_l \subseteq CH(\widehat{F}_l(\mathcal{V}_l, \mathbf{W}_l, \boldsymbol{\theta}_l))$. In other words, every output of a layer $l$ associated with an input in $\mathcal{X}_l$ is in the convex hull of $\widehat{F}_l(\mathcal{V}_l, \mathbf{W}_l, \boldsymbol{\theta}_l)$
\label{prop:SAPNM_layer}
\end{proposition}

\begin{proof}[\textbf{Proof}] The mapping of a layer $l$ from $F$ is composed of an affine map ($\mathbf{W}_l \mathbf{x}_l + \boldsymbol{\theta}_l$) and a non-linear map ($ReLU$). Hence, since $\mathcal{X}_l$ is a closed convex polytope and $\mathcal{V}_l$ its vertices, the affine map of $\mathcal{X}_l$ is given by the convex hull of the affine map of its vertices, as computed by  Algorithm \ref{alg:affine_map}.

As previously presented, the $ReLU$ map is non-linear and therefore does not necessarily preserve the convexity of a given input set. However, as established by Proposition \ref{prop:relu_convexa_por_ortant}, the $ReLU$ mapping preserves the convexity of a convex set inside a given orthant.

Therefore, we divide the resulting set of the affine map in such a way that each partition is inside a single orthant, so that we can apply the $ReLU$ map to the set $\widehat{\mathcal{X}}_l = CH(AM(\mathcal{V}, \mathbf{W}_l, \boldsymbol{\theta}_l))$.
   As assured by Proposition \ref{prop:intersection_part2}, the intersection of a given orthant $\mathcal{O}_j$ with the polytope $\widehat{\mathcal{X}}_l$ consists of the convex hull of the union of the vertices of $\widehat{\mathcal{X}}_l$ that are in $\mathcal{O}_j$, with the vertices from the intersection of the edges of $\widehat{\mathcal{X}}_l$ with the supporting hyperplanes that define $\mathcal{O}_j$. 
Note that $\mathcal{O}_j$ represents a given orthant, for all $j \in \{1, \ldots, 2^{|\mathbf{x}_l|}\}$.

Thus, we first need to compute the edges of $\widehat{\mathcal{X}}_l$, calculated by $EI(AM(\mathcal{V}, \mathbf{W}_l, \boldsymbol{\theta}_l))$, as stated by Proposition \ref{prop:adjacent_vertices}, followed by the determination of the intersection of these edges with the supporting hyperplanes of $\mathcal{O}_j$, given by $II(AM(\mathcal{V}_l, \mathbf{W}_l, \boldsymbol{\theta}_l),EI(AM(\mathcal{V}_l, \mathbf{W}_l, \boldsymbol{\theta}_l)))$ and computed by Algorithm \ref{alg:orthant_intersection}.

Therewith, we have that $\widehat{\mathcal{X}}_l$ was divided in such a way that each partition is inside a single orthant.
Then, the $ReLU$ mapping is applied to the vertices of each partition of $\widehat{\mathcal{X}}_l$, resulting in $ReLU(II(AM(\mathcal{V}_l, \mathbf{W}_l, \boldsymbol{\theta}_l),EI(AM(\mathcal{V}_l, \mathbf{W}_l, \boldsymbol{\theta}_l))))$. 
Proposition \ref{prop:relu_convexa_por_ortant} assures that the $ReLU$ mapping preserves the convexity inside a single orthant, which allows its previous application.
Note that all the vertices will be in the non-negative orthant after applying the $ReLU$ mapping.

Finally, we compute the convex hull of all the output sets of the $ReLU$ mapping (at most $2^{|\mathbf{x}_l|}$) by removing those points that are not vertices of $CH(ReLU(II(AM(\mathcal{V}_l, \mathbf{W}_l, \boldsymbol{\theta}_l),EI(AM(\mathcal{V}_l, \mathbf{W}_l, \boldsymbol{\theta}_l)))))$. This final step is  computed by  Algorithm \ref{alg:removing_internal_points} (RP) as stated in Equation \eqref{eq:convex_layer_mapping}.

Consequently,  $CH(\widehat{F}_l(\mathcal{V}_l, \mathbf{W}_l, \boldsymbol{\theta}_l))$ is the convex hull of the exact map $\mathcal{Y}_l$ of $\mathcal{X}_l$.
\end{proof}

We denote by $\widehat{F}$ the composition of $\widehat{F}_l$ for all layers of the neural network, except the last one, where the non-linear map is not applied, given by $\widehat{F}(\mathcal{V}, \mathbf{W}, \boldsymbol{\theta}) = (\widehat{F}_{L} \circ \widehat{F}_{L-1} \circ \cdots \circ \widehat{F}_{1})(\mathcal{V}, \mathbf{W}, \boldsymbol{\theta})$. 
   Furthermore, we define $\mathcal{Y} = \{F(\mathbf{x}) \mid \mathbf{x} \in \mathcal{X}\}$ as the exact output set of the network, regarding the  closed convex input polytope $\mathcal{X}$ and its corresponding set of vertices $\mathcal{V}$.

\begin{proposition} Given a closed convex polytope $\mathcal{X}$ as the input set and $\mathcal{V}$, the set of its vertices, then we have that $\mathcal{Y} \subseteq CH(\widehat{F}(\mathcal{V}, \mathbf{W}, \boldsymbol{\theta}))$. Put in different terms, each output of the network associated with an input in $\mathcal{X}$ is in $\widehat{F}(\mathcal{V}, \mathbf{W}, \boldsymbol{\theta})$.
\label{prop:SAPNM_network}
\end{proposition}

\begin{proof}[\textbf{Proof}] As established by Proposition  \ref{prop:SAPNM_layer}, $\mathcal{Y}_l \subseteq CH(\widehat{F}_l(\mathcal{V}_l, \mathbf{W}_l, \boldsymbol{\theta}_l))$ for a given layer $l$ of the neural network $F$. 
Then, for $l=1$, it follows that:
\begin{equation}
    \mathcal{Y}_1 \subseteq CH(\widehat{F}_1(\mathcal{V}_1, \mathbf{W}_1, \boldsymbol{\theta}_1))
\end{equation}
where $\mathcal{V}_1$ is the set of vertices from $\mathcal{X}_1$ and $\mathcal{X}_1 = \mathcal{X}$. As the output of the first layer is the input of the second one, we have that  $\mathcal{X}_2 = CH(\widehat{F}_1(\mathcal{V}_1, \mathbf{W}_1, \boldsymbol{\theta}_1))$ and that $\mathcal{V}_2 = \widehat{F}_1(\mathcal{V}_1, \mathbf{W}_1, \boldsymbol{\theta}_1)$.

For $l=2$:
\begin{equation}
    \mathcal{Y}_2 \subseteq CH(\widehat{F}_2(\mathcal{V}_2, \mathbf{W}_2, \boldsymbol{\theta}_2))
\end{equation}
Then, by replacing $\mathcal{V}_2$ with the output of layer 1, results in:
\begin{equation}
    \mathcal{Y}_2 \subseteq CH(\widehat{F}_2(\widehat{F}_1(\mathcal{V}_1, \mathbf{W}_1, \boldsymbol{\theta}_1), \mathbf{W}_2, \boldsymbol{\theta}_2))
\end{equation}

Now, for layers $k$ and $k+1$, it follows by induction that:
\begin{equation}
    \mathcal{Y}_{k+1} \subseteq CH(\widehat{F}_{k+1}(\widehat{F}_k(\mathcal{V}_k, \mathbf{W}_k, \boldsymbol{\theta}_k), \mathbf{W}_{k+1}, \boldsymbol{\theta}_{k+1}))
\end{equation}
Consequently, the mapping  $\hat{F}$ in fact computes an over-approximation for the exact output set $\mathcal{Y}$.
\end{proof}

\subsection{Correctness of EPNM algorithm}

The set $\mathcal{Y}_l$ denotes the exact output associated with the input set $\mathcal{X}_l$, regarding the layer $l$ of the neural network $F$.
The mapping of a given layer $l$ implemented by Algorithm \ref{alg:layer_map_2}, denoted here by $\widehat{E}_l$, is stated as:
\begin{equation}
    \widehat{E}_{l,k}(\mathcal{V}_l, \mathbf{W}_l, \boldsymbol{\theta}_l) = ReLU\left(SP\left(OS\bigl(II(AM(\mathcal{V}_l, \mathbf{W}_l, \boldsymbol{\theta}_l),EI(AM(\mathcal{V}_l, \mathbf{W}_l, \boldsymbol{\theta}_l)))\bigr)\right)_k\right)
    \label{eq:exact_layer_mapping}
\end{equation}
where $AM$ is the affine map computed by Algorithm \ref{alg:affine_map}, 
   $EI$ is the edge identification implemented by Algorithm \ref{alg:edge_identification},
$II$ is the intersection identification  given by Algorithm \ref{alg:orthant_intersection}, 
    $ReLU$ is computed by Algorithm \ref{alg:relu_map}, 
$OS$ is implemented by Algorithm \ref{alg:origin_search} to verify if the origin is in the polytope, and 
    finally $SP$, which separates the vertices in orthants, is computed by Algorithm \ref{alg:orthant_separation}. 
Notice that $k \in \{1, \ldots, K\}$ is the index that represents each $\widehat{E}_l(\mathcal{V}_l, \mathbf{W}_l, \boldsymbol{\theta}_l)$. 
   We denote the convex hull mapping for a given input set of vertices by $CH$.

\begin{proposition} Given a convex closed polytope $\mathcal{X}_l$ as input set and $\mathcal{V}_l$, the set of its vertices, we have that $\mathcal{Y}_l = \bigcup_{k=1}^K CH(\widehat{E}_{l,k}(\mathcal{V}_l, \mathbf{W}_l, \boldsymbol{\theta}_l))$, where $K$ is the number of sets that comprise $\widehat{E}_{l}(\mathcal{V}_l, \mathbf{W}_l, \boldsymbol{\theta}_l)$. Put another way, the set of outputs of the layer $l$, resulting from all inputs in $\mathcal{X}_l$, consists of the union of the convex hull of each set $\widehat{E}_{l,k}(\mathcal{V}_l, \mathbf{W}_l, \boldsymbol{\theta}_l)$.
\label{prop:EPNM_layer}
\end{proposition}

\begin{proof}[\textbf{Proof}] As presented previously, the mapping of a given layer $l$ of the neural network $F$ is a composition of two different functions: one affine mapping ($\mathbf{W}_l \mathbf{x}_l + \boldsymbol{\theta}_l$) with one non-linear mapping ($ReLU$). As $\mathcal{X}_l$ is a closed convex polytope, the affine mapping is obtained by the convex hull of the affine map of each of its vertices, implemented by Algorithm \ref{alg:affine_map}.

$ReLU$ does not necessarily preserve the convexity of a given input set. However, as shown by  Proposition \ref{prop:relu_convexa_por_ortant}, the $ReLU$ mapping preserves the convexity of an input set if the input is inside a single orthant.

Note that the affine mapping of the input set is computed by $AM(\mathcal{V}_l, \mathbf{W}_l, \boldsymbol{\theta}_l)$, where $\widehat{\mathcal{X}}_l = CH(AM(\mathcal{V}_l, \mathbf{W}_l, \boldsymbol{\theta}_l))$ denotes the application of the affine mapping in $\mathcal{X}_l$.
   Therefore, it is necessary to split $\widehat{\mathcal{X}}_l$ in such a way that each partition lies inside a single orthant, so that it becomes possible to apply the $ReLU$ mapping to the vertices of $\widehat{\mathcal{X}}_l$. 
As assured by Proposition \ref{prop:intersection_part2}, the intersection of an orthant $\mathcal{O}_j$ with a polytope $\widehat{\mathcal{X}}_l$ consists of the union of the vertices of $\widehat{\mathcal{X}}_l$ that are in $\mathcal{O}_j$, with the vertices in the intersection of the edges of $\widehat{\mathcal{X}}_l$ with the supporting hyperplanes of $\mathcal{O}_j$ and the origin, if the latter lies inside $\widehat{\mathcal{X}}_l$, according to Proposition \ref{prop:pertinencia_origem}.

Thus, we firstly compute the edges of $\widehat{\mathcal{X}}_l$, a step denoted by $EI(AM(\mathcal{V}_l, \mathbf{W}_l, \boldsymbol{\theta}_l))$ and computed by Algorithm \ref{alg:edge_identification}, as stated by Proposition \ref{prop:adjacent_vertices}. Then, the intersection of these edges with the supporting hyperplanes of orthant $\mathcal{O}_j$ is obtained with Algorithm \ref{alg:orthant_intersection}, and finally Algorithm \ref{alg:origin_search} verifies whether the origin belongs to  $\widehat{\mathcal{X}}_l$.
The result of such an operation is given by:
\begin{equation}
    OS\bigl(II(AM(\mathcal{V}_l, \mathbf{W}_l, \boldsymbol{\theta}_l),EI(AM(\mathcal{V}_l, \mathbf{W}_l, \boldsymbol{\theta}_l)))\bigr)
\end{equation}

In the next step, we separate the vertices of $\widehat{\mathcal{X}}_l$ in different sets, such that those vertices in the same set represent the portion of $\widehat{\mathcal{X}}_l$ that is inside a single orthant. 
This process takes place to enable the application of the $ReLU$ mapping, as the separation allows the non-linear mapping to be applied in each partition of $\widehat{\mathcal{X}}_l$ while ensuring convexity, as established by Proposition \ref{prop:relu_convexa_por_ortant}.
Algorithm \ref{alg:orthant_separation} performs the partitioning operation, given by:
\begin{equation}
    SP\left(OS\bigl(II(AM(\mathcal{V}_l, \mathbf{W}_l, \boldsymbol{\theta}_l),EI(AM(\mathcal{V}_l, \mathbf{W}_l, \boldsymbol{\theta}_l)))\bigr)\right)
\end{equation}

The result of $SP$ is the set of sets of vertices, where the convex hull of each set represents a partition of $\widehat{\mathcal{X}}_l$ restricted to a single orthant. Thus, $\widehat{E}_l(\mathcal{V}_l, \mathbf{W}_l, \boldsymbol{\theta}_l)$ denotes the result of the application of the $ReLU$ mapping to each of these sets. 
Finally, we have that $\widehat{E}_l(\mathcal{V}_l, \mathbf{W}_l, \boldsymbol{\theta}_l)$ represents the exact mapping of $\mathcal{X}_l$, as we applied both the affine and the non-linear mapping without any over-approximation. Consequently, Algorithm \ref{alg:layer_map_2} in fact returns the exact output set for a given layer $l$ of the neural network $F$.
\end{proof}

Finally, we denote by $\widehat{E}$ the composition of $\widehat{E}_l$, for each $l \in \{1, \ldots, L\}$. 
To avoid repetition, the Proposition \ref{prop:EPNM_network} is not presented. However, it follows the same inductive process as presented for the APNM (Proposition \ref{prop:SAPNM_network}).

\begin{proposition} Given a closed convex polytope $\mathcal{X}$ as input set and $\mathcal{V}$, the set of its vertices, we have that $\mathcal{Y} = \bigcup_{k=1}^K CH(\widehat{E}_{k}(\mathcal{V}, \mathbf{W}, \boldsymbol{\theta}))$, where $K$ is the number of sets in $\widehat{E}(\mathcal{V}, \mathbf{W}, \boldsymbol{\theta})$. In other terms, the set of outputs of the neural network $F$, associated with the input set $\mathcal{X}$, is in the union of the convex hull of each set $\widehat{E}_{k}(\mathcal{V}, \mathbf{W}, \boldsymbol{\theta})$.
\label{prop:EPNM_network}
\end{proposition}

\section{Application}
\label{sec:applications}

In this section we present a comparative analysis between the proposed vertex-based reachability approach and representative algorithms from the literature. This comparison is carried out by verifying one of the properties from ACAS XU \cite{julian2019deep}.

\subsection{ACAS XU}

The Aircraft Collision Avoidance System (ACAS XU) \cite{julian2019deep} comprises a set of fully connected neural networks that aim to eliminate the possibility of collisions between two aircraft. The system comprises 45 trained models, where each model receives as inputs five properties of the ownship and the intruder (the aircraft that is invading the space of the ownship): the distance from the ownship to the intruder ($\rho$), the angle to the intruder regarding the ownship heading direction ($\theta$), the heading angle of intruder relative to ownship heading direction ($\psi$), the speed of the ownship ($v_{ownship}$), and the speed of the intruder ($v_{intruder}$). We can see a representation of the inputs in Figure \ref{fig:acas_xu_visualization}.

\begin{figure}[!ht]
    \centering
    \includegraphics[scale=0.5]{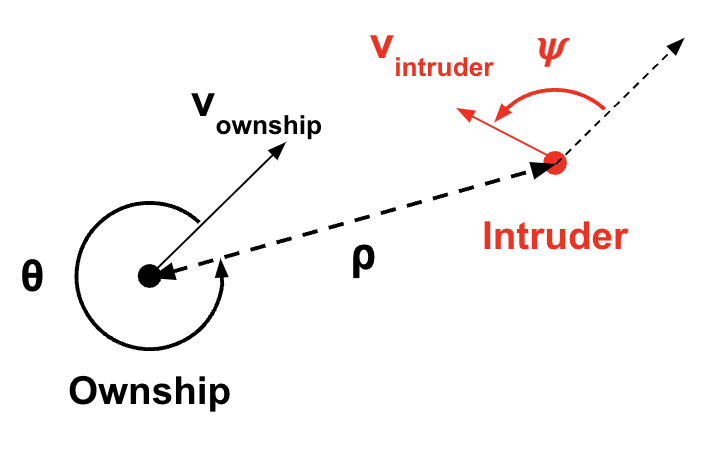}
    \caption{Representation of the ACAS XU inputs. The black dot represents the position of the ownship and the red dot the position of the intruder. $\rho$ represents the Euclidean distance between both aircrafts, $\theta$ the angle between the ownship heading direction and the vector that connects both aircraft, and $\psi$ the angle between the ownship and the intruder heading direction.}
    \label{fig:acas_xu_visualization}
\end{figure}

There are two extra parameters that are not used as inputs to the neural network. The first one is the time until loss of vertical separation ($\tau$), whereas the second one is the previous prediction advice ($a_{previous}$). 
These $\tau$ and $a_{previous}$ parameters are discretized such that for each possible combination of values for $\tau$ and $a_{previous}$, a different model is trained.
This process resulted in a total of 45 trained models, as mentioned before. Each trained model associates an input pattern to five possible categories: 
clear of conflict ($y_0$), weak left ($y_1$), weak right ($y_1$), strong left ($y_3$), and strong right ($y_4$). Figure \ref{fig:acas_xu_network} contains a simple representation of a single the neural network classifier associated with a value for $\tau$ and $a_{previous}$. Each neural network comprises 6 hidden layers, containing 50 neurons each. For further details on the training and prediction process of these models, please refer to \cite{julian2019deep}.

\begin{figure}[!ht]
    \centering
    \includegraphics[scale=0.4]{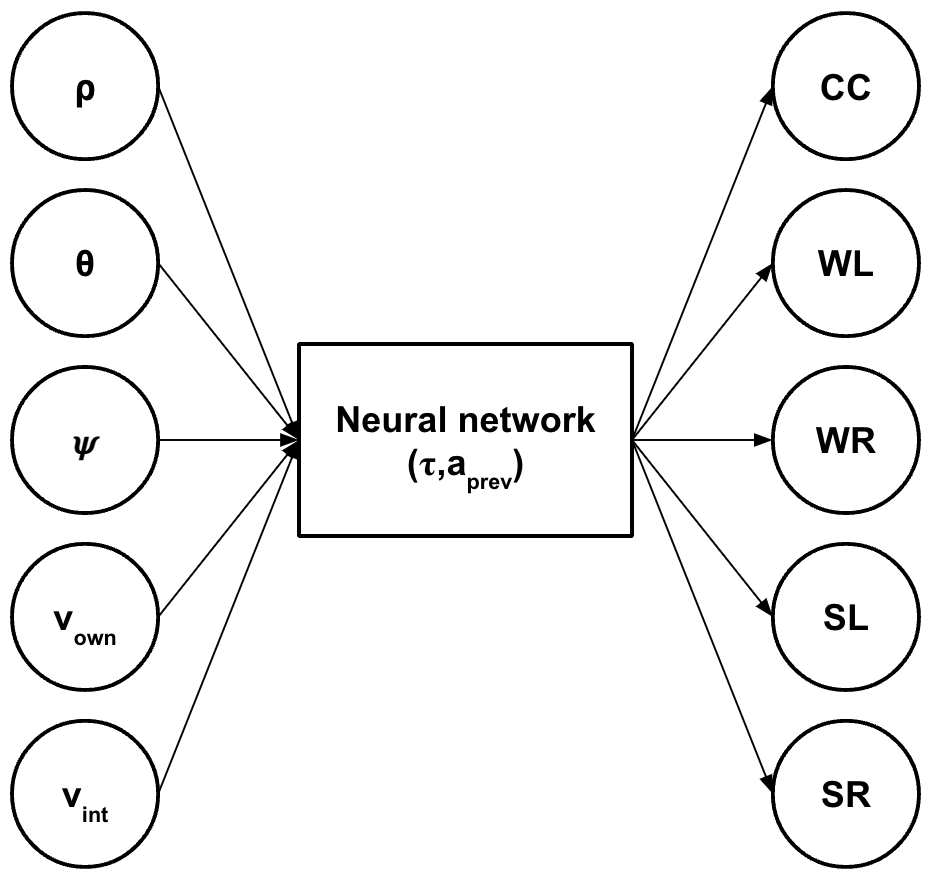}
    \caption{Representation of the inputs and the outputs of a single neural network from ACAS XU. On the left of the neural network we present each of the five expected inputs: $\rho$, $\theta$, $\psi$, $\mathbf{v}_{own}$ and $\mathbf{v}_{int}$. The neural network is defined according to a given $\tau$ and a $a_{prev}$. To the right of the neural network we have the expected probability for each output class. $CC$ stands for clear of conflict, $WL$ for weak left, $WR$ for weak right, $SL$ for strong left, and $SR$ for strong right.}
    \label{fig:acas_xu_network}
\end{figure}

Based on the trained models of the ACAS XU, ten desired properties have been proposed so that this system works correctly according to its crafted design. 
In this paper, to compare with existing verification approaches, we verified Property \ref{property:1}, formally stated as: 

\begin{property} \label{property:1} The conditions are established as follows:
\begin{itemize}
    \item input constraint: $\rho \geq 55947,691$,  $v_{ownship} \geq 1145$ and $v_{intruder} \leq 60$; 
    \item output constraint: $y_0 \leq 1500$;
\end{itemize}
where $y_0$ is the output associated with the \textit{clear of conflict} output class. 
\end{property}

Recall from the problem statement that the reachability analysis aims to verify a reachable set $\mathcal{R}$, obtained from $\mathcal{X}$, such that $\mathcal{R} \cap \lnot \mathcal{Y} = \emptyset$ holds. $\mathcal{Y}$ denotes the expected output for the inputs in $\mathcal{X}$.
For the Property \ref{property:1}, we have that:
\begin{eqnarray}
    \mathcal{X} = \{(\rho, \theta, \phi, v_{ownship}, v_{intruder}) \mid && 55947.691 \leq \rho \leq 60760, \notag \\
    && -\pi \leq \theta \leq \pi, \notag \\
    && -\pi \leq \phi \leq \pi, \notag \\
    && 1145 \leq v_{ownship} \leq 1200, \notag \\
    && 0 \leq v_{intruder} \leq 60\} 
\end{eqnarray}
and that:
\begin{eqnarray}
    \mathcal{Y} = \{(y_0, y_1, y_2, y_3, y_4) \mid y_0 \leq 1500\} 
\end{eqnarray}

Property \ref{property:1} is the only one presented in this paper because only this property is evaluated by means of comparison with other formal verification algorithms. 
However, all of the remaining properties are formally described in \cite{katz2017reluplex}.

\subsection{Experimental description}

We present in this section the procedures of our experimental setup. The experimental results are divided into two main parts: the first part aiming to validate and compare the results with algorithms from the literature, and the second part to evaluate the features of our approaches. For the first part, we conducted the experiments by:
\begin{enumerate}
    \item Generate the vertices of the input polytope, based on the input constraint imposed by Property \ref{property:1};
    \item Set a timeout of 24 hours for each algorithm verification;
    \item Compute the output reachable set and the verification status for each reachability algorithm (Exact polytope network mapping (EPNM), MaxSens \cite{xiang2018output}, Ai2 \cite{gehr2018ai2} and ExactReach \cite{xiang2017reachable});
    \item Compute counterexamples and the verification status for the search algorithm (Reluplex \cite{katz2017reluplex}, Duality \cite{dvijotham2018dual}, MIPVerify \cite{tjeng2017evaluating} and NSVerify \cite{lomuscio2017approach}); and
    \item Compare results.
\end{enumerate}

For the second experiment, which aims to analyze the parallelism behavior of the algorithm, the experimental setup consists of: 
\begin{enumerate}
    \item Generate the vertices of the input polytope, based on the input constraint imposed by Property \ref{property:1};
    \item Set the number of parallel processes, denoted by $p$, such that $p \in \{1,4,8,12,16,20,24,28,32\}$; and
    \item Verify Property \ref{property:1} for each $p$.
\end{enumerate}

\subsection{Hardware and software specification}

For comparative matters, we provide the specification for both hardware resources and software language. 
The comparative experiment previously presented was performed in an Intel Xeon CPU E5-2630 V4 of 2.20GHz, with 40 available CPU. 
    Those algorithms from the literature and the algorithms proposed in this work were developed in Julia language. 
The implementation of the algorithms from the literature were available on \cite{neuralverificationrepo}.
The implementation of our algorithms is available on \cite{vertexbasedreachabilityanalysis}.

\subsection{Comparative results}

The validation and comparative results of the verification of ACAS Xu models for Property \ref{property:1} are presented in two different perspectives: firstly among those verification procedures that follow reachability approaches, then comparing with approaches that make use of different techniques (search or optimization). Finally, we present some useful features of the proposed approach.

\subsubsection{EPNM versus reachability approaches}

By comparing the EPNM approach with those verification algorithms that follow a reachability approach, as presented in Table \ref{tab:rechability_results}, it can be seen that the proposed exact approach verified most of the neural networks within the stipulated timeout time (43 out of 45 neural networks). 
As we can see, none of the other existing exact approaches were able to verify a single model within a day of execution (24 hours). 
Notice also that the approximate approaches (MaxSens and Ai2) finished their execution, though, due to their over-approximation, these procedures did not estimate  the correct status well (which is acceptable, as these approaches are sound but not complete).

\begin{table}[!ht]
    \centering
        \caption{Comparative results between the proposed approach (EPNM) regarding existing reachability approaches from the literature. These results refer to the verification of Property \ref{property:1} of ACAS XU models.}
    \label{tab:rechability_results}
    \begin{tabular}{c|c|c|c|c}
       \multicolumn{1}{c}{} & \multicolumn{1}{|c}{Proposed} &  \multicolumn{3}{|c}{Approaches from the Literature}  \\   \cline{2-5}
        Status & EPNM & ExactReach & MaxSens & Ai2\\ \hline
        holds & 43 & - & - & - \\ \hline
        violated & - & - & 45 & 45\\ \hline
        timeout & 2 & 45 & - & -
    \end{tabular}

\end{table}

\subsubsection{EPNM versus search and optimization approaches}

In comparison to optimization and search approaches, the proposed EPNM approach also reached interesting results. Compared to Reluplex results from the literature, EPNM could verify more neural networks within the specified timeout. However, by executing Reluplex in the same hardware conditions of EPNM, the verification was not completed within 24 hours. The same occurred for the NSVerify procedure. 

\begin{table}[!ht]
    \centering
        \caption{Comparative results between the proposed approach regarding existing search and optimization approaches from the literature. These results refer to the verification of Property \ref{property:1} of ACAS XU models.}
    \label{tab:experimental_result}
    \begin{tabular}{c|c|c|c|c|c}
       \multicolumn{1}{c}{}        & \multicolumn{1}{|c}{Proposed} &  \multicolumn{4}{|c}{Approaches from the Literature}  \\   \cline{3-6}
        Status & EPNM & Reluplex\footnote[1]{} & Reluplex (24 hours)\footnote[2]{} & Duality & NSVerify\\ \hline
        holds & 43 & 41 & - & - & - \\ \hline
        violated & - & - & - & - & - \\ \hline
        timeout & 2 & 4 & 45 & - & 45\\ \hline
        unknown & - & - & - & 45 & -
    \end{tabular}

\end{table}

\footnotetext[1]{Results extracted from \cite{katz2017reluplex}.}
\footnotetext[2]{Results from the authors for a $24$-hour execution.}

\subsubsection{Parallel Computation}

Due to the characteristics of the proposed approaches, their implementation allows the parallelization in a procedural level.
We chose to implement the parallelization in two of the procedures (EI and II), because the remaining algorithms did not respond well due to the tradeoff between the overhead and the speed-up of the parallelization. 

The first one is the edge identification (EI) algorithm. 
    For this procedure, we created a pool for the execution with the size equal to the number of available threads. 
The pool guarantees that, after each thread ends its execution, a new thread is started and takes the empty space.
   For each vertex, a new thread was initiated, which verified if the adjacency property holds for the current and each of the remaining vertices. 
As sharing memory was not necessary, because each thread has its own adjacency list, no synchronization approach was implemented.
By the end of the execution, the algorithm concatenated the adjacency list calculated for each vertex.

The intersection identification (II) was the second parallelized procedure.
     Following the same idea, a new thread was created for each vertex. 
After the initialization is completed, the procedure identifies those intersection points between the current and each of its adjacent vertices.
     In this case, similarly to the previous one, no synchronization was necessary, as each thread carried its own list of intersections. 
At the end of the pool execution, the procured concatenated those intersection points associated with each vertex.

The results of the second part of the experiments are depicted in Figure \ref{fig:duration_x_threads}. 
As can be seen in this figure, as the number of available threads for the execution increases, the running time decreases significantly. 
This characteristic of the algorithm can be explored for huge problems.

Figure \ref{fig:speedup_x_threads} presents the speedup behavior for the algorithm EPNM. As expected, the algorithm indeed reduce the runtime as there is an increment on the available threads, though the difference between the real and the ideal curve indicates that the parallel processes are not ideally balanced (there are threads waiting for some execution to end).


\begin{figure}[!ht]
     \centering
     \begin{subfigure}[b]{0.4\textwidth}
         \centering
        \includegraphics[width=\textwidth]{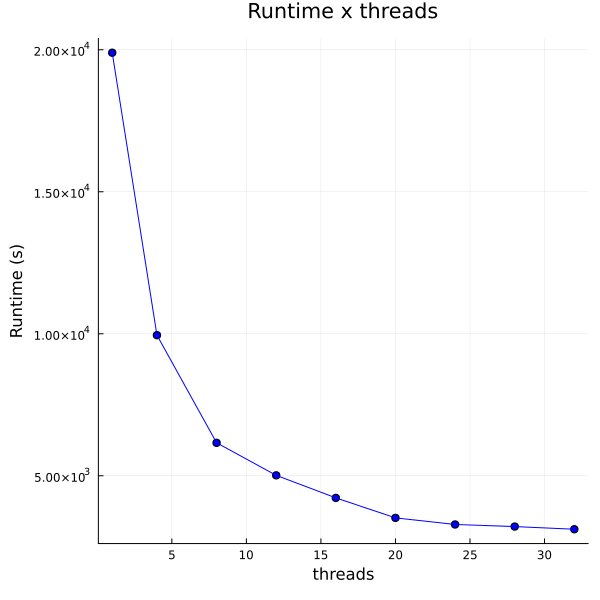}
        \caption{
        The graph illustrates the execution speed-up obtained from parallelizing the proposed algorithm. As the number of available threads increases, there is a significant reduction in the running time (duration).
        }
        \label{fig:duration_x_threads}
     \end{subfigure}
     \hfill
     \begin{subfigure}[b]{0.4\textwidth}
         \centering
         \includegraphics[width=\textwidth]{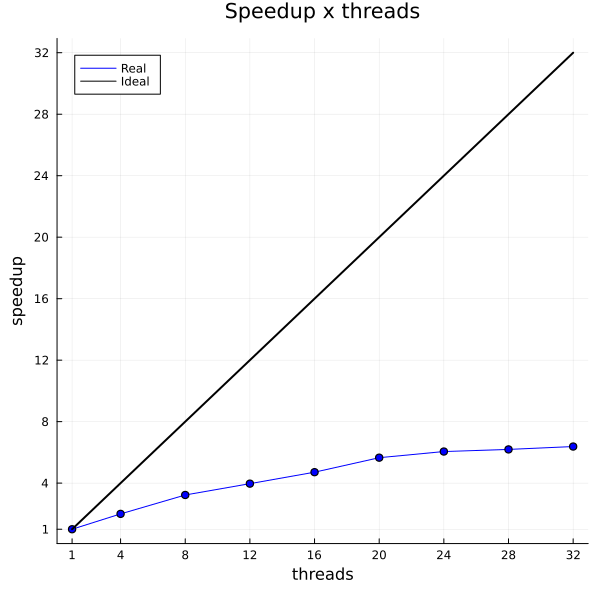}
         \caption{Comparison between the actual speed up of the EPNM and the ideal speed up.}
         \label{fig:speedup_x_threads}
     \end{subfigure}
    \caption{Illustrative visualization of the EPNM parallelization behavior. We present both, the runtime and the speed up for the algorithm EPNM execution on a single ACAS Xu model for the property \ref{property:1}.}
    \label{fig:speedup_behavior}
\end{figure}

\begin{figure}[!ht]
     \centering
     \begin{subfigure}[b]{0.4\textwidth}
         \centering
         \includegraphics[width=\textwidth]{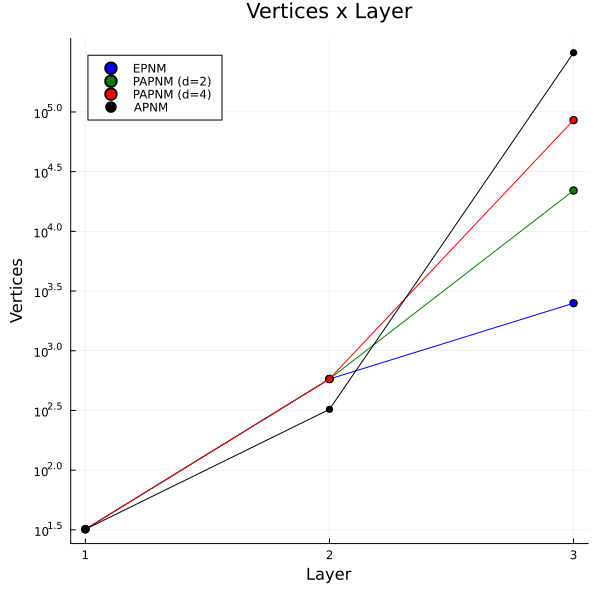}
         \caption{Total of number vertices processed at each layer of a single neural network. The y-axis is in logarithmic scale. The APNM algorithm presents a significantly higher increment in the number of vertices after layer $3$, in comparison to PAPNM and EPNM. Both PAPNM and EPNM  executions have the same value at layer $2$, as expected, though EPNM has a lower number of vertices at layer 3.}
         \label{fig:vertices_behavior}
     \end{subfigure}
     \hfill
     \begin{subfigure}[b]{0.4\textwidth}
         \centering
         \includegraphics[width=\textwidth]{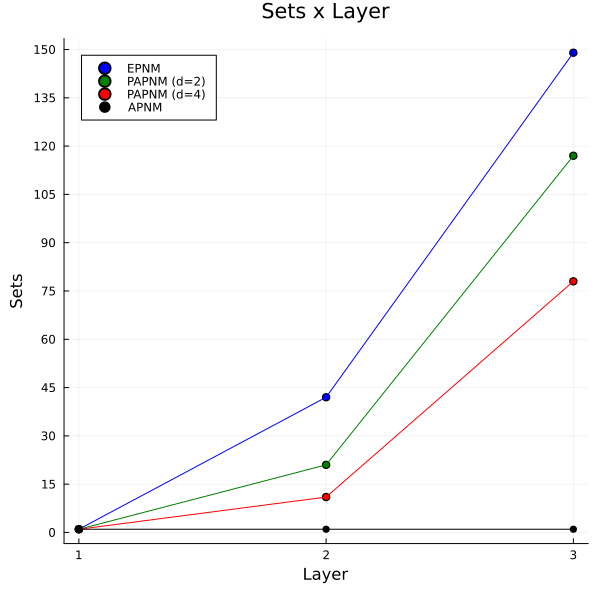}
         \caption{Total number of sets generated at each layer processed by each of the algorithms. As expected, APNM keeps $1$ set at the end of every layer execution. The EPNM has the greatest increment on the number of sets, which is an expected behavior.}
         \label{fig:sets_behavior}
     \end{subfigure}
    \caption{Illustration of the complexity behavior of APNM, EPNM and PAPNM with $d=2$ and $d=4$ for the first $3$ layers of a single neural network from ACAS XU model.}
    \label{fig:complexity_behavior}
\end{figure}

\subsubsection{Complexity behavior of the proposed approaches}

Our experiments showed that, differently from the expected, the algorithm EPNM has a shorter running time in comparison to APNM, for the ACAS XU model.
   This behavior can be explained considering the total number of vertices that are processed in both algorithms.
Figure \ref{fig:complexity_behavior} presents the behavior of both, the total number of vertices and sets processes at each layer of a single neural network from ACAS XU. 

\begin{table}[!htb]
    \caption{Results on the running time of EPNM, APNM and PAPNM with $d=2$ and $d=4$ for the first three layers of a single neural network.}
    \begin{subtable}{.48\linewidth}
      \centering
        \caption{Behavior of the total number of vertices processed for each layer of a neural network by APNM, EPNM and PAPNM.}
        \begin{tabular}{c|c|c|c|c}
            Layer & APNM & PAPNM & PAPNM & EPNM \\ 
             & & $(d=2)$ & $(d=4)$ & \\\hline
            1 & 32 & 32 & 32 & 32 \\ \hline
            2 & 323 & 580 & 580 & 580 \\ \hline
            3 & 313286 & 85404 & 21977 & 2502
        \end{tabular}
    \end{subtable}%
    \hfill
    \begin{subtable}{.48\linewidth}
      \centering
        \caption{Behavior of the total number of sets processed for each layer of a neural network by APNM, EPNM and PAPNM.}
        \begin{tabular}{c|c|c|c|c}
            Layer & APNM & PAPNM & PAPNM & EPNM \\ 
             & & $(d=2)$ & $(d=4)$ & \\\hline
            1 & 1 & 1 & 1 & 1 \\ \hline
            2 & 1 & 11 & 21 & 42 \\ \hline
            3 & 1 & 78 & 117 & 149
        \end{tabular}
    \end{subtable} 
    \label{tab:complexity_behavior}
\end{table}

Figure \ref{fig:vertices_behavior} shows that, from the same input set, the algorithm APNM generates the greater set of vertices for representing its approximation of the output set compared to both, EPNM and PAPNM (\textit{i.e.}, $313286$ vertices at the third layer).
    Figure \ref{fig:sets_behavior} depicts the increment on the number of sets for EPNM and PAPNM.
Table \ref{tab:complexity_behavior} reports the data used to create Figure \ref{fig:complexity_behavior}.

Both results lead to a lower average number of vertices across each of the sets for EPNM. 
On the other hand, the opposite occurs to APNM  which has a single set with a strongly increasing number of vertices.
Table \ref{tab:vertices_x_sets} shows the average number of vertices per set for each algorithm at each layer of the neural network. 
This means that the merging process (performed partially by PAPNM and completely by APNM) induces a simplification on the total number of sets, though as a side effect it significantly increases  the total number of vertices to be processed.

\begin{table}[!htbp]
    \centering
    \caption{Average of the total number of vertices within each set for each algorithm. The results are presented layer by layer.}
    \begin{tabular}{c|c|c|c|c}
        Layer & APNM & PAPNM & PAPNM & EPNM \\ 
         & & $(d=2)$ & $(d=4)$ & \\\hline
        1 & 32 & 32 & 32 & 32 \\ \hline
        2 & 323 & 52.7 & 27.6 & 13.8 \\ \hline
        3 & 313286 & 1094.9 & 187.8 & 16.8
    \end{tabular}
    \label{tab:vertices_x_sets}
\end{table}

The number of vertices within a set is directly related to the running time for each algorithm because the computational complexity of each of the procedures that comprise APNM, PAPNM and EPNM is directly related to the total of vertices processed.

\section{Conclusion}
\label{sec:conclusion}

In this work, we proposed two vertex-based reachability algorithms for formal verification of deep neural networks. These algorithms compute a reachable output set, which may consist of a set of polyhedral sets, for a given input polyhedral set, satisfying different properties: the first one (APNM) computes an approximation for the output reachable set, while the second one (EPNM) computes the exact output reachable set.

Supported by formal demonstrations of correctness, the proposed algorithms were shown to correctly verify properties of neural networks with the $ReLU$ activation function. More specifically, it was shown that APNM yields an overestimation of the output reachable set, while EPNM computes the exact reachable set.

Our proposal was applied to a benchmark problem for neural network verification and compared to some of the algorithms previously proposed in the literature. The results showed that among the verification algorithms that make use of reachability analysis, the presented EPNM approach concluded most of the verifications (43 out of 45 neural networks), differently from the ExactReach, which is another exact approach from the literature that could not verify any neural network within the specified timeout. Compared to those algorithms that are not complete, despite the fact that these approaches were able to finish their executions, the expected output was not reached in any case (see Table \ref{tab:rechability_results}).

In comparison to those methods that make use of optimization and search strategies, the result reported in the literature for Reluplex surpassed ours in terms of running time. However, under the same hardware conditions, our approach was able to overcome their results.

Finally, we argue that the algorithms proposed in this paper are strong candidates for the verification of neural networks with $ReLU$ activation. Additionally, the running time can be reduced drastically by using multiple processing cores. Future work includes the investigation of a different construction for the search approach in EPNM and of a third approach that can be designed to improve performance by using heuristics for the reduction of sets and vertices during the search process.

\section*{Acknowledgment}

This work was funded in part by Funda{\c{c}}{\~a}o de Amparo {\`a} Pesquisa e Inova{\c{c}}{\~a}o do Estado de Santa Catarina (FAPESC) under grant 2021TR2265.

\bibliography{refs}
\bibliographystyle{ieeetr}

\end{document}